\documentclass{article}

\PassOptionsToPackage{numbers, compress}{natbib}
\usepackage[preprint]{neurips_2025}

\usepackage[utf8]{inputenc} % allow utf-8 input
\usepackage[T1]{fontenc}    % use 8-bit T1 fonts
\usepackage{hyperref}       % hyperlinks
\usepackage{url}            % simple URL typesetting
\usepackage{booktabs}       % professional-quality tables
\usepackage{amsfonts}       % blackboard math symbols
\usepackage{nicefrac}       % compact symbols for 1/2, etc.
\usepackage{microtype}      % microtypography
\usepackage{xcolor}         % colors
\usepackage{amsmath,amssymb,amsfonts}
\usepackage{amsthm}
\usepackage{algorithm}
\usepackage{graphicx}
\usepackage{textcomp}
\usepackage{booktabs}
\usepackage{subcaption}
\usepackage[noend]{algpseudocode}
\usepackage{csquotes}
\usepackage{multirow}
\usepackage{wrapfig}
\usepackage{tabularx}
\newtheorem{proposition}{Proposition}

\title{Wasserstein Barycenter Soft Actor-Critic}

\author{%
  Zahra.~Shahrooei \\
  Department of Mechanical Engineering\\
  Rochester Institute of Technology\\
  Rochester, NY 14623 \\
  \texttt{zs9580@rit.edu} \\
  \And
  Ali.~Baheri \\
  Department of Mechanical Engineering\\
  Rochester Institute of Technology\\
  Rochester, NY 14623 \\
  \texttt{akbeme@rit.edu} \\
}

\begin{document}

\maketitle

\begin{abstract}

Deep off-policy actor-critic algorithms have emerged as the leading framework for reinforcement learning in continuous control domains. However, most of these algorithms suffer from poor sample efficiency, especially in environments with sparse rewards. In this paper, we take a step towards addressing this issue by providing a principled directed exploration strategy. We propose \textit{Wasserstein Barycenter Soft Actor-Critic} (WBSAC) algorithm, which benefits from a pessimistic actor for temporal difference learning and an optimistic actor to promote exploration. This is achieved by using the Wasserstein barycenter of the pessimistic and optimistic policies as the exploration policy and adjusting the degree of exploration throughout the learning process. We compare WBSAC with state-of-the-art off-policy actor-critic algorithms and show that WBSAC is more sample-efficient on MuJoCo continuous control tasks.
\end{abstract}

\section{Introduction}
Deep reinforcement learning (DRL) has found applications across diverse fields, including autonomous driving, robotics, healthcare, finance \cite{deng2016deep}, and gaming \cite{kiran2021deep,hambly2023recent}. Among the various DRL methodologies, actor-critic frameworks \cite{silver2014deterministic, lillicrap2015continuous, fujimoto2018addressing} have demonstrated significant success in continuous control tasks and have gained widespread adoption within the control systems and robotics fields. Despite their successes, these methods still face the challenge of high sample complexity, which results from maximizing the lower bound of the Q-function to prevent overestimation and reliance on basic exploration techniques. In practice, the state-of-the-art actor-critic algorithms, such as deep deterministic policy gradient (DDPG) \cite{lillicrap2015continuous} and its successor, twin delayed DDPG (TD3) \cite{fujimoto2018addressing}, inject symmetric noise directly into the action space to encourage exploration. In contrast, soft actor-critic (SAC) encourages stochasticity through entropy regularization. However, these heuristic exploration approaches are not sample efficient and can limit performance, particularly in sparse-reward or high-dimensional environments \cite{hao2023exploration}.

Moving beyond simple noise injection or basic entropy regularization towards more directed or uncertainty-aware exploration strategies, several approaches have been proposed for actor-critic methods to enhance the performance without causing overestimation and instability \cite{zheng122018self,lee2021sunrise,januszewski2021continuous,likmeta2023wasserstein,ciosek2019better}. Most of these efforts employ the principle of optimistic in the face of uncertainty (OFU) \cite{ziebart2010modeling} and quantify epistemic uncertainty as a signal for exploring promising areas. These works rely on a conservative actor that outputs a pessimistic policy for temporal difference learning and consider an optimistic policy to facilitate exploration. To prevent instability and avoid sub‑optimal behavior, the optimistic policy is typically regularized by a Kullback–Leibler (KL) divergence constraint that keeps it close to the pessimistic policy. The other existing works use multiple actors \cite{zhang2019ace,lyu2022efficient,li2023multi,kumar2018predicting,ren2021probabilistic,nauman2025decoupled,yang2024efficient,whitney2021decoupled, li2023realistic} to more explicitly decouple the task of temporal difference learning from exploration. While these works have shown successful performance in improving the sample efficiency, establishing a balance between exploration and exploitation is still needed to avoid divergence and suboptimal performance.

\begin{wrapfigure}{r}{0.5\textwidth}
    \centering
    \includegraphics[width=0.5\textwidth]{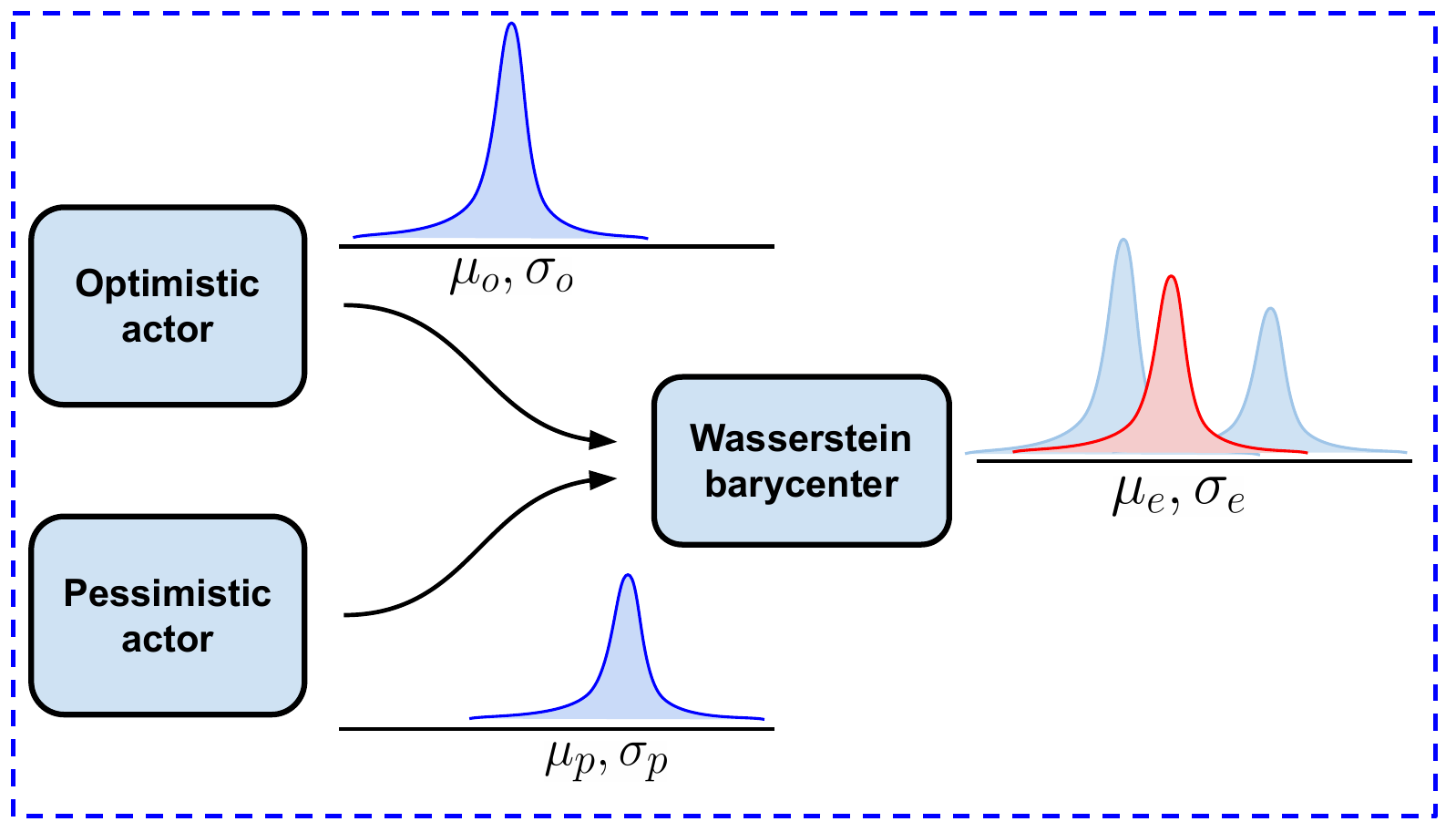} % Adjust width as needed
    \caption{\scriptsize WBSAC uses Wasserstein barycenter of optimistic and pessimistic policies as the exploration policy.}
    \label{fig:exp}
\end{wrapfigure}

To mitigate the aforementioned limitations, we present Wasserstein barycenter soft actor-critic (WBSAC), which uses two distinct actors with different objectives. The optimistic actor is trained to maximize the Q-function upper bound, while the pessimistic actor maximizes a Q-function lower bound. Unlike many approaches that rely on KL-divergence constraints to keep exploration close to a target policy, WBSAC forms the exploration policy as the Wasserstein barycenter of distinct pessimistic and optimistic policies, as shown in Figure~\ref {fig:exp}. This allows for a controlled and geometrically meaningful balance, where the influence of the optimistic policy gradually increases, rather than being strictly constrained. This adaptive balance improves exploration capability and performance without introducing overestimation bias. Using Wasserstein barycenter for blending optimistic and pessimistic strategies and controlling the degree of exploration is the novelty of our work. The main contributions of this paper are:

\begin{itemize}
  \item We present WBSAC, an off‐policy actor–critic framework with dual actors to enhance exploration capability and performance. WBSAC makes the fusion of pessimistic and optimistic strategies possible in a unique way and allows for a controlled, progressive shift in exploration strategy, starting from conservative behavior towards an optimistic strategy.
  
  \item We theoretically show that WBSAC enjoys high differential entropy, which leads to more effective learning and improved state-action coverage.
  \item We empirically demonstrate that WBSAC outperforms state‐of‐the‐art baselines in both sample efficiency and overall performance on several Mujoco tasks and sparse reward case studies.
\end{itemize}

\section{Related work}
\label{sec:Evaluation}

Significant research has been devoted to optimistic exploration strategies to improve sample efficiency by using Q-value upper bounds. These approaches include analytically derived transformations of pessimistic policies or employing separate exploration-focused actor networks. Directed exploration approaches that use an optimistic policy
for exploration and a conservative policy for
exploitation, provide different strategies for blending pessimism with optimism \cite{zheng122018self,lee2021sunrise,januszewski2021continuous,likmeta2023wasserstein,ciosek2019better,moskovitz2021tactical}. For instance, optimistic actor-critic (OAC) \cite{ciosek2019better} models the optimistic policy by combining conservative
policy with the linear approximation of the Q-function upper bound, and tactical optimism and pessimism (TOP) \cite{moskovitz2021tactical} determines the appropriate level of optimism/pessimism during training by framing this selection process as a multi-armed bandit problem. In these works, both optimistic and pessimistic policies are extracted from a conservative actor, and by the KL divergence constraint, the exploration policy stays close to the pessimistic policy. WBSAC is different from these works in both modeling the optimistic policy with a separate actor and the exploration policy. Furthermore, unlike WBSAC, both TOP and Wasserstein actor-critic (WAC) \cite{likmeta2023wasserstein} use a distributional critic.

Approaches with multiple actors are similar to WBSAC \cite{nauman2025decoupled,lyu2022efficient,li2023multi, pan2020softmax,lyu2023value, chen2024double,nauman2024bigger}. WBSAC distinguishes itself from decoupled actor-critic (DAC) \cite{nauman2025decoupled}, which optimizes the optimistic actor towards a KL-regularized Q-value upper-bound by introducing the Wasserstein barycenter of the policies as the exploration policy. Unlike WBSAC, which concentrates on stochastic policies, DARC \cite{lyu2022efficient} implements the dual actor setup for deterministic policies. It promotes exploration by selecting the action associated with the higher Q‑value from two separate actors and reduces estimation bias through critic regularization. However, basing action selection on the maximal Q‑value can amplify errors arising from inaccurate value estimates. Bigger, regularized, optimistic (BRO) \cite{nauman2024bigger} uses dual policy optimistic exploration and non-pessimistic quantile Q-value approximation to make a balance between exploration and exploitation. Multi-actor mechanism (MAM) \cite{li2023multi} proposes a multi-actor framework with multiple action choices within a single state to optimize policy selection.

Optimal transport has been integrated to RL in several studies \cite{baheri2023risk,queeney2023optimal,metelli2019propagating,shahrooei2024optimal, baheri2025wave,baheri2024synergy,baheri2023understanding}. For instance, \citet{metelli2019propagating} proposes a novel approach called Wasserstein Q-learning (WQL), which considers Bayesian posterior Q and target distributions and applies Wasserstein barycenters to model and propagate uncertainty. Their method demonstrated improved exploration and faster learning in tabular domains compared to classic RL algorithms. \citet{likmeta2023wasserstein} extends this idea to actor-critic algorithms and suggests a distributional critic, which minimizes the Wasserstein distance between Q and target distributions. \citet{baheri2025wave} introduces an actor-critic algorithm called Wasserstein adaptive value estimation (WAVE), which enhances stability through an adaptive critic regularization term based on the Wasserstein distance between consecutive Q-distributions.

\section{Preliminaries}
\label{sec:prelim}

\noindent{\textbf{Markov Decision Processes and Q-functions. }}A Markov Decision Process (MDP) provides a formal framework for sequential decision-making problems in RL. An MDP is defined by the tuple $\mathcal{M} = (\mathcal{S}, \mathcal{A}, P, R, \gamma)$, where $\mathcal{S}$ is the set of states, $\mathcal{A}$ is the set of actions, $P : \mathcal{S} \times \mathcal{A} \to \Delta(\mathcal{S})$ is the state transition probability function (with $\Delta(\mathcal{S})$ being the set of probability distributions over $\mathcal{S}$), $R : \mathcal{S} \times \mathcal{A} \to \mathbb{R}$ is the reward function, and $\gamma \in [0,1)$ is the discount factor.  The behavior of an agent is described by a policy $\pi : \mathcal{S} \to \Delta(\mathcal{A})$, which maps states to a probability distribution over actions. The performance of a policy is often evaluated using value functions. The action-value function (or Q-function) for a policy $\pi$, denoted $Q^{\pi}(s, a)$, is the expected discounted return starting from state $s$, taking action $a$, and subsequently following policy $\pi$: $Q^{\pi}(s, a) = \mathbb{E}_{\pi} \left[ \sum_{t=0}^{\infty} \gamma^t R(s_t, a_t) \mid s_0 = s, a_0 = a \right]$. The primary goal in many RL problems is to find an optimal policy $\pi^*$ that maximizes the expected return. This optimal policy satisfies $\pi^*(s) = \arg\max_{a \in \mathcal{A}} Q^*(s,a)$ for all $s \in \mathcal{S}$, where $Q^*(s,a)$ is the optimal action-value function.

\noindent{\textbf{Soft Actor-Critic. }}Soft actor-critic (SAC) \cite{haarnoja2018soft} is an off-policy actor-critic algorithm based on the maximum entropy RL framework \cite{ziebart2008maximum}. This framework encourages exploration by augmenting the expected return objective with an entropy term for the policy. SAC employs a stochastic policy $\pi_\phi$, typically a neural network parameterized by $\phi$. To mitigate overestimation bias, it utilizes two separate Q-function networks, $Q_{\theta_1}$ and $Q_{\theta_2}$ (parameterized by $\theta_1$ and $\theta_2$), and uses the minimum of their predictions for policy improvement and target calculation. Correspondingly, two target Q-networks, $\hat{Q}_{\hat{\theta}_1}$ and $\hat{Q}_{\hat{\theta}_2}$, are maintained and updated via polyak averaging of the Q-network parameters $\hat{\theta}_i \leftarrow \tau \theta_i + (1-\tau)\hat{\theta}_i$, where $\tau \ll 1$. The Q-functions $Q_{\theta_i}$ are trained to minimize the soft Bellman residual. For each critic $Q_{\theta_i}$, the loss function is:
\begin{equation}
\mathcal{L}_Q(\theta_i) = \mathbb{E}_{(s_t, a_t, r_t, s_{t+1}) \sim \mathcal{D}}\left[\left(Q_{\theta_i}(s_t, a_t) - y_t \right)^2\right], \quad i \in \{1,2\}
\end{equation}
where the target value $y_t = r_t + \gamma \left( \min_{j=1,2} \hat{Q}_{\hat{\theta}_j}(s_{t+1}, a'_{t+1}) - \alpha \log \pi_\phi(a'_{t+1}|s_{t+1}) \right), \quad a'_{t+1} \sim \pi_\phi(\cdot|s_{t+1})$ where $\alpha$ is the entropy temperature coefficient, and $\mathcal{D}$ is the replay buffer. The policy $\pi_\phi$ is optimized to maximize the expected future return, including the entropy term. This is achieved by minimizing the following loss:

\begin{equation}
\mathcal{L}_{\pi}(\phi)
= \mathbb{E}_{s_t \sim \mathcal{D},\, a_t \sim \pi_\phi}
\Bigl[
\alpha\,\log\pi_\phi(a_t \mid s_t)
\;-\;
\min_{j=1,2}Q_{\theta_j}(s_t, a_t)
\Bigr]
\end{equation}

The entropy temperature $\alpha$ can be automatically tuned to balance the reward and entropy terms. The loss function for $\alpha$ is:
\begin{equation}
\mathcal{L}_{\alpha}(\alpha) = \mathbb{E}_{s_t \sim \mathcal{D}, a_t \sim \pi_\phi(\cdot|s_t)} \left[ -\alpha (\log \pi_\phi(a_t|s_t) + \mathcal{H}_0) \right]
\end{equation}
where $\mathcal{H}_0$ is the target entropy.

\noindent{\textbf{Wasserstein Barycenter. }}The Wasserstein distance is a key metric from optimal transport theory to quantify the difference between probability distributions. For a metric space \((\mathcal{X}, d)\) equipped with the Borel \(\sigma\)-algebra \(\mathcal{B}(\mathcal{X})\), let \(\mu\) and \(\nu\) be probability measures defined on this space. Given a proper cost function \(d(x, y)\), the \(p\)-Wasserstein distance \cite{villani2009optimal} between \(\mu\) and \(\nu\) is defined as:
\begin{equation}
W_p(\mu, \nu) = \left( \inf_{\pi \in \Pi(\mu, \nu)} \int_{\mathcal{X} \times \mathcal{X}} d(x, y)^p \, d\pi(x, y) \right)^{\frac{1}{p}}, \quad p \geq 1
\end{equation}
where \(\Pi(\mu, \nu)\) is the set of all couplings (joint probability measures) on the product space \(\mathcal{X} \times \mathcal{X}\) with marginals \(\mu\) and \(\nu\). In this work, we take the cost function to be the Euclidean norm, $d(x,y) = \|x - y\|_2$, and set $p = 2$. Building upon this notion of distance, Wasserstein barycenters are widely used for capturing the average of probability distributions in a geometrically meaningful way. Given a set of probability measures \(\{\mu_i\}_{i=1}^n\) defined on a metric space \((\mathcal{X}, d)\), the Wasserstein barycenter \cite{agueh2011barycenters} represents a central probability distribution that minimizes the weighted sum of Wasserstein distances to each distribution in the set. Precisely, given weights \(\{\xi_i\}_{i=1}^n\) satisfying \(\sum_{i=1}^n \xi_i = 1\) and \(\xi_i \geq 0\), the Wasserstein barycenter \(\bar{\mu}\) is defined as:
\begin{equation}
\bar{\mu} = \arg \min_{\mu} \sum_{i=1}^n \xi_i W_2(\mu, \mu_i)^2
\end{equation}

\section{Wasserstein barycenter soft
actor-critic}
\label{sec:WBSAC}
Our WBSAC algorithm maintains the core components of SAC, including dual critic networks and their targets for Q-value estimation, but distinguishes itself by employing two specialized actor networks and dynamically blending their outputs via Wasserstein barycenter to form the exploration policy.

\begin{algorithm*}[!t]
\caption{Wasserstein Barycenter Soft
Actor-Critic (WBSAC) \label{alg:WassersteinAC}}
\textbf{Input:} exploration schedule parameter $\lambda$, target update rate $\tau$, learning rates $\alpha_{\theta}$, $\alpha_{\pi_o}$, $\alpha_{\pi_p} $, $\alpha_\alpha$, uncertainty bonus parameter $\beta_o$\\
Initialize critic networks $Q_{\theta_1}$, $ Q_{\theta_2}$ with random parameters $\theta_1$, $\theta_2$\\
Initialize target networks $\theta_{1}^{'}\leftarrow \theta_{1}$, $\theta_{2}^{'}\leftarrow \theta_{2}$\\
Initialize actor networks with random parameters $\phi$ and $\varphi$\\
Initialize replay buffer $\mathcal{D}$
\begin{algorithmic}[1]
\For{each iteration}
    \For{each environment step}
        % \State Compute Wasserstein barycenter of $\pi_p(\phi, s_t)$ and $\pi_o(\varphi, s_t)$
        \State Obtain $\pi_p(\cdot|s_t)$ from pessimistic actor network
        \State Obtain $\pi_o(\cdot|s_t)$ from optimistic actor network
        \State Compute exploration policy $\pi_e(\cdot|s_t)$ according to Eq. \ref{explorationpolicy} with weights $\xi_p, \xi_o$
        \State $a_t \sim \pi_e(a_t | s_t)$
        \Comment Sample action from the exploration policy
        \State $s_{t+1} \sim p(s_{t+1} | s_t, a_t)$
        \Comment Sample transition from the environment
        \State $\mathcal{D} \leftarrow \mathcal{D} \cup \{(s_t, a_t, r(s_t, a_t), s_{t+1})\}$
        \Comment Store the transition in the replay buffer
        \State Update Wasserstein barycenter weights $\xi_o$ and $\xi_p$
    \EndFor

    \For{each gradient step}
        \State Sample mini-batch of $N$ transitions $(s_t, a_t, r(s_t, a_t), s_{t+1})$ from $\mathcal{D}$
        \State $\theta_i \leftarrow \theta_i - \alpha_{\theta} \nabla_{\theta_i} \mathcal{L}_Q(\theta_i) \quad \text{for} \quad i \in \{1, 2\}$
        \Comment{Update critic parameters}
        \State $\varphi \leftarrow \varphi - \alpha_{\pi_o} \nabla_\varphi \mathcal{L}_{\pi_o}(\varphi)$
        \Comment{Update optimistic policy parameters}
        \State $\phi \leftarrow \phi - \alpha_{\pi_p} \nabla_\phi \mathcal{L}_{\pi_p}(\phi)$
        \Comment{Update pessimistic policy parameters}
        \State $\alpha \leftarrow \alpha - \alpha_\alpha \nabla_\alpha \mathcal{L}_{\alpha}(\alpha)$
        \Comment{Update entropy temperature}
        \State $\theta'_i \leftarrow \tau \theta_i + (1 - \tau) \theta'_i, \quad \text{for } i=1,2$
        \Comment{Update target critic parameters}
        
    \EndFor
\EndFor
\end{algorithmic}
\textbf{Output:} Optimized parameters $\theta_1$, $\theta_2$, $\phi$, $\varphi$, $\alpha$
\end{algorithm*}
 
\noindent{\textbf{Pessimistic and Optimistic Policies. }}We first define a \emph{pessimistic actor}, parameterized by $\phi$, which relies on the lower bound of the Q-function estimates to ensure robust performance. Its objective is formulated as:
\begin{equation}
\mathcal{L}_{\pi_p}(\phi) = \mathbb{E}_{s \sim \mathcal{D},\, a_t \sim \pi_p} \left[ \alpha \log \pi_p(a_t|s_t) - \min_{i=1,2} Q_{\theta_i}(s_t, a_t) \right]
\end{equation}
This actor outputs a pessimistic policy denoted by $\pi_{p}\bigl(\cdot \mid s_{t}\bigr)$. Additionally, WBSAC incorporates an \emph{optimistic actor}, parameterized by $\varphi$, which explores actions with potentially high long-term rewards using discrepancies between the Q-value predictions of two critic networks. High divergence in critic estimates often indicates regions of the action space with greater uncertainty, where exploration may uncover high-reward actions. This optimistic actor targets these regions by maximizing an objective that balances the average Q-value with a standard deviation term:
\begin{equation}
\mathcal{L}_{\pi_o}(\varphi) = \mathbb{E}_{s \sim \mathcal{D},\, a_t \sim \pi_o(\cdot | s_t; \varphi)} \left[ -\left( \mu_Q + \beta_o \cdot \sigma_Q \right) \right]
\end{equation}

where $\mu_Q = \frac{Q_{\theta_1}(s_t, a_t) + Q_{\theta_2}(s_t, a_t)}{2} \quad$ and $\sigma_Q = \sqrt{\frac{1}{2} \sum_{i=1}^2 (Q_i - \mu_Q)^2}$ and $\beta_o$ is a positive hyperparameter, which controls the uncertainty bonus. This actor outputs an optimistic policy specified by $\pi_{o}\bigl(\cdot \mid s_{t}\bigr)$. 

% $\quad \sigma_Q = \frac{ \left| Q_{\theta_1}(s_t, a_t) - Q_{\theta_2}(s_t, a_t) \right| }{2}$, 

\noindent{\textbf{Wasserstein Barycentric Exploration Policy. }}Since SAC employs stochastic policies by outputting probability distributions over the action space, we propose integrating the pessimistic and optimistic policies through the Wasserstein barycenter. Specifically, we define the \emph{exploration policy} $\pi_e$ as the distribution that minimizes the weighted sum of Wasserstein distances to $\pi_p$ and $\pi_o$:
\begin{equation}
\pi_e = \arg \min_{\pi} \left( \xi_p \, W_2(\pi, \pi_p)^2 + \xi_o \, W_2(\pi, \pi_o)^2 \right)
\label{explorationpolicy}
\end{equation}
where $\xi_p$ and $\xi_o$ are non-negative respective weights of pessimistic and optimistic policies, which satisfy $\xi_p + \xi_o = 1$. By tuning  $\xi_p$ and $\xi_o$, we ensure a controlled balance between exploitation and exploration. 

% This formulation provides a gradual shift from pessimistic to optimistic behavior. Hence, 

\noindent{\textbf{Exploration Degree Adjustment. }}To ensure stable learning, WBSAC initially biases $\pi_e$ towards the pessimistic policy $\pi_p$. This is achieved by scheduling the weight $\xi_o$ to start at zero and grow linearly as training progresses. The weight scheduling strategy is considered $\xi_o = \min\left(1,\, \lambda \cdot \frac{t}{T}\right)$, where $\lambda$ is the exploration schedule parameter, $t$ denotes the current environment timestep, and $T$ represents the total environment time steps.

\noindent{\textbf{Closed-Form Barycenter for Factorized Gaussian Policies. }}For Gaussian policies, the barycenter has a closed-form solution, which decreases computation cost while preserving theoretical validity. Assuming both the pessimistic policy $\pi_p(\cdot|s) = \mathcal{N}(\mu_p(s), \Sigma_p(s))$ and the optimistic policy $\pi_o(\cdot|s) = \mathcal{N}(\mu_o(s), \Sigma_o(s))$) are Gaussian, their $W_2$-barycenter $\pi_e(\cdot|s) = \mathcal{N}(\mu_e(s), \Sigma_e(s))$ is also Gaussian. Its mean $\mu_e(s)$ and covariance $\Sigma_e(s)$ are efficiently computed in closed form:
\begin{align}
\mu_e(s) &= \xi_p \mu_p(s) + \xi_o \mu_o(s) \label{eq:barycenter_mean_wbsac} \\
\Sigma_e(s) &= \left(\xi_p \Sigma_p(s)^{1/2} + \xi_o \Sigma_o(s)^{1/2}\right)^2 \label{eq:barycenter_cov_wbsac}
\end{align}
where $\Sigma^{1/2}$ is the principal matrix square root of a positive semi-definite covariance matrix $\Sigma$. Figure ~\ref{fig:wrap-exploration} illustrates how the exploration policy gradually transitions toward the optimistic policy as a function of the evolving Wasserstein barycenter in one-dimensional action space in a fixed state.

Algorithm~\ref{alg:WassersteinAC} summarizes the overall WBSAC procedure. During environment interaction, actions are sampled from the exploration policy $\pi_e$, which is dynamically formed as the Wasserstein barycenter of the pessimistic and optimistic policies using scheduled weights. In the training phase, the critic networks, both actor networks, and the SAC temperature parameter $\alpha$ are updated using transitions sampled from the replay buffer.

\begin{wrapfigure}{r}{0.45\textwidth}  % 'r' for right; adjust width as needed
  \centering
  \vspace{-1em}  % Optional: shift up if needed
  \begin{minipage}{0.35\textwidth}
    \centering
    \includegraphics[width=0.95\linewidth]{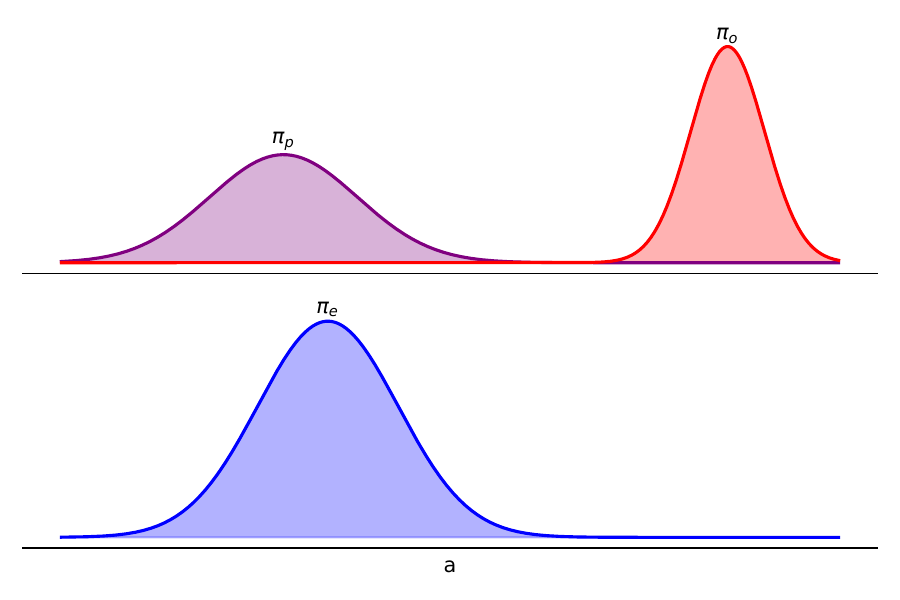}
    \caption*{\scriptsize $\xi_o = 0.1$}
    \vspace{0.5em}
    \includegraphics[width=0.95\linewidth]{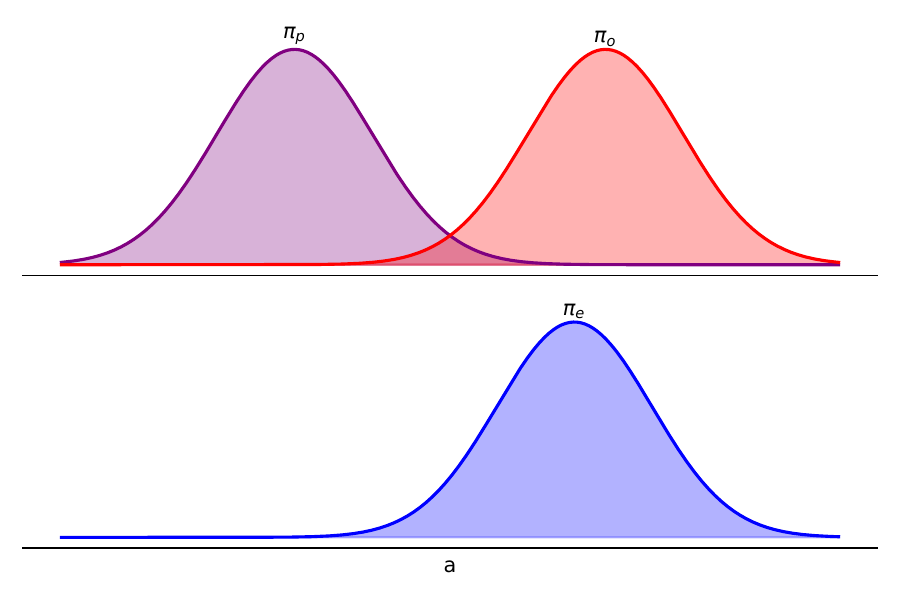}
    \caption*{\scriptsize $\xi_o = 0.9$}
  \end{minipage}
  \caption{\scriptsize Exploration policy evolution. Initially, the exploration policy $\pi_e$ aligns with the pessimistic policy $\pi_p$. As learning proceeds, the exploration policy shifts toward the optimistic policy $\pi_o$.}
  \label{fig:wrap-exploration}
\end{wrapfigure}

\textbf{Remark:} Encouraging the agent to explore the parts of the action space that the two critics assign high value or disagree on their value after the exploitation phase, avoids the problem of unidirectional exploration in SAC, which was introduced by \cite{ciosek2019better} first. In other words, the inclusion of the standard deviation term (critic disagreement) in the optimistic actor's objective explicitly encourages exploration towards regions of higher epistemic uncertainty, rather than uniformly sampling the upper bound of Q-value estimates.

\textbf{Remark:} By linking the exploration policy to the pessimistic policy at the beginning of the learning process, we prioritize exploitation and ensure that the transitions sampled from the exploration policy align
with the probabilities expected under the pessimistic policy. As the training procedure proceeds, we encourage the agent to a small degree of exploration to escape the local optimum associated with lower-bound estimations (if any). The gradual interpolation from purely pessimistic to increasingly optimistic alleviates the phenomenon known as \emph{pessimistic underexploration}, first discussed in \cite{ciosek2019better}, in which an agent constrained by overly cautious value estimates can become trapped in suboptimal regions of the state–action space. WBSAC’s weighted‐barycenter mechanism guides the policy beyond such a local optimum.

\textbf{Proposition 1. }\textit{For factorized Gaussian pessimistic and optimistic policies, the exploration policy $\pi_e$ (derived from Eq.~\ref{eq:barycenter_mean_wbsac} and \ref{eq:barycenter_cov_wbsac}) has its differential entropy, $H(\pi_e(s))$, lower-bounded for any state $s \in \mathcal{S}$ as:}
\begin{equation}
\label{eq:entropy_inequality}
H\!\bigl(\pi_e(s)\bigr)\;\;\ge\;\;
\xi_p\,H\!\bigl(\pi_p(s)\bigr)+\xi_o\,H\!\bigl(\pi_o(s)\bigr).
\end{equation}
The proof is deferred to the Appendix. B.

\textbf{Remark:} Proposition 1 provides a theoretical basis for the exploration capability of $\pi_e$, which is essential for achieving broad coverage and collecting diverse data. As the weight $\xi_o$ for the optimistic policy increases, its contribution to this lower bound ensures $\pi_e$ maintains a capacity for broad action selection. This stochasticity promotes the generation of diverse data, as the agent samples a wider variety of actions, and leads to enhancing coverage of the state-action space.

\begin{table}[t!]
\centering
\caption{Average return over the last $10$ evaluations over $5$ trials across MuJoCo environments.}
\label{tab:results}
\resizebox{\columnwidth}{!}{%
\begin{tabular}{lcccc}
\toprule
\textbf{Task} & \textbf{WBSAC (ours)} & \textbf{SAC} & \textbf{DARC} & \textbf{OAC} \\
\midrule
Ant-v5         & $5564.8 \pm 188.5$   & $5306.5 \pm 418.3$   & $3926.2 \pm 904.2$   & $\mathbf{5867.0} \pm 292.3$ \\
HalfCheetah-v5 & $\mathbf{6466.4} \pm 1411.0$ & $5409.7 \pm 2219.1$ & $4234.1 \pm 1768.3$ & $5210.5 \pm 2353.8$ \\
Walker2d-v5    & $\mathbf{4417.0} \pm 703.0$  & $3939.1 \pm 262.0$  & $3760.8 \pm 485.3$  & $4112.9 \pm 217.4$ \\
Humanoid-v5    & $\mathbf{5701.2} \pm 463.4$  & $5086.5 \pm 170.0$  & $4545.5 \pm 780.8$  & $4850.1 \pm 279.1$ \\
\bottomrule
\end{tabular}}
\end{table}

\section{Experiments}
\label{sec:Evaluation}

We evaluate WBSAC performance on MuJoCo continuous control environments \cite{todorov2012mujoco}, DeepMind control suite tasks \cite{tassa2018deepmind}, and the PointMaze navigation task \cite{de2023gymnasium}. We additionally provide a sensitivity analysis of hyperparameters. In addition to vanilla SAC, our comparative analysis includes DARC framework \cite{lyu2022efficient}, which incorporates two actor networks to mitigate the overestimation bias by taking the minimum of value estimates from two actors and the underestimation bias by taking the maximum of the minimum value estimates from double critics.  Our third baseline is OAC \cite{ciosek2019better}, which is on top of SAC and focuses on directed exploration to enhance sample efficiency.
% DARC introduces critic regularization to reduce the uncertainty in value estimates from double critics by a convex combination of value estimates from the two actors. Furthermore, it employs a cross-update scheme where only one actor-critic pair is updated per timestep, with the other pair used for value correction. This delayed update promotes policy smoothing, similar to TD3 algorithm.

\noindent\textbf{Performance on MuJoCo benchmarks.} We evaluate WBSAC performance on four MuJoCo continuous control tasks \cite{brockman2016openai} via OpenAI Gym \cite{brockman2016openai}, and provide averaged results over five seeds. In MuJoCo benchmark environments, the maximum number of interaction steps is $1000$ per episode. We evaluate the performance of the pessimistic policy every $5,000$ step. We use the moving average window of $40$ iterations.

\begin{table}[!b]
\centering
\caption{Average return over the last $10$ evaluations over $5$ trials across DeepMind control suite tasks.}
\label{tab:dmc_results}
\resizebox{\columnwidth}{!}{%
\begin{tabular}{lcccc}
\toprule
\textbf{Domain-Task} & \textbf{WBSAC (ours)} & \textbf{DARC} & \textbf{SAC} & \textbf{OAC} \\
\midrule
Finger-Turn Easy    & \textbf{943.95 $\pm$ 17.55} & 937.73 $\pm$ 27.25 & 892.81 $\pm$ 85.34 & 937.47 $\pm$ 15.37 \\
Finger-Turn Hard    & \textbf{910.96 $\pm$ 60.29} & 784.17 $\pm$ 138.09 & 856.08 $\pm$ 90.36 & 891.46 $\pm$ 39.44 \\
Ball-in-cup Catch   & \textbf{983.93 $\pm$ 2.29}  & 980.12 $\pm$ 2.67  & 982.60 $\pm$ 2.00  & \textbf{983.90 $\pm$ 2.09} \\
Hopper Hop    & \textbf{130.14 $\pm$ 116.82} & 120.33 $\pm$ 75.24    & 116.33 $\pm$ 50.66  & 63.00 $\pm$ 49.23 \\
Cheetah Run         & \textbf{740.16 $\pm$ 41.33} & 685.53 $\pm$ 75.57 & 647.55 $\pm$ 26.77 & 672.05 $\pm$ 38.82 \\
Walker Walk         & \textbf{965.34 $\pm$ 6.31}  & 852.31 $\pm$ 108.17 & 852.31 $\pm$ 115.41 & 965.80 $\pm$ 2.85 \\
Walker Run          & \textbf{684.25 $\pm$ 18.10} & 527.50 $\pm$ 77.67 & 577.86 $\pm$ 59.49 & 671.08 $\pm$ 41.42 \\
Humanoid Run        & \textbf{144.06 $\pm$ 11.55} & 1.23 $\pm$ 0.21    & 22.88 $\pm$ 36.82  & 98.77 $\pm$ 50.70 \\
\bottomrule
\end{tabular}
}
\end{table}

Figure~\ref{fig:return} presents the average total episode reward and its standard deviation for five seeds. We observe that using two actors enhances the overall performance. As can be seen, our algorithm outperforms SAC in all four case studies. We emphasize that no hyperparameter tuning was performed on the Humanoid task. Table~\ref{tab:results} provides the average return and standard deviation of the last ten evaluations. As observed, in all case studies except for Ant-v5, WBSAC achieves superior or comparable performance to OAC. 

\begin{figure*}[!t]
  \centering
  \begin{subfigure}{0.24\textwidth}
    \includegraphics[width=\linewidth]{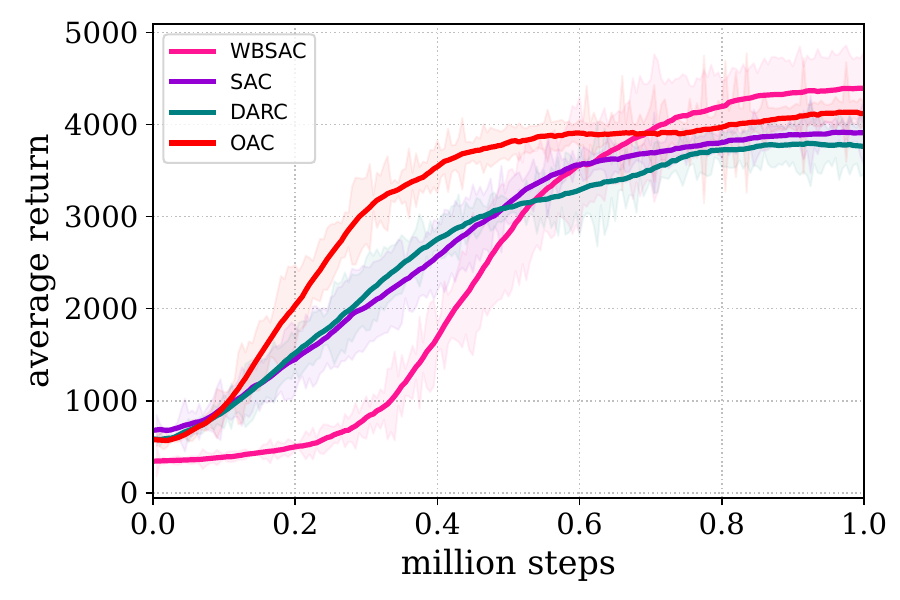}
    \caption{Walker2d}
    \label{fig:subfig-1}
  \end{subfigure}
  \hfill
  \begin{subfigure}{0.24\textwidth}
    \includegraphics[width=\linewidth]{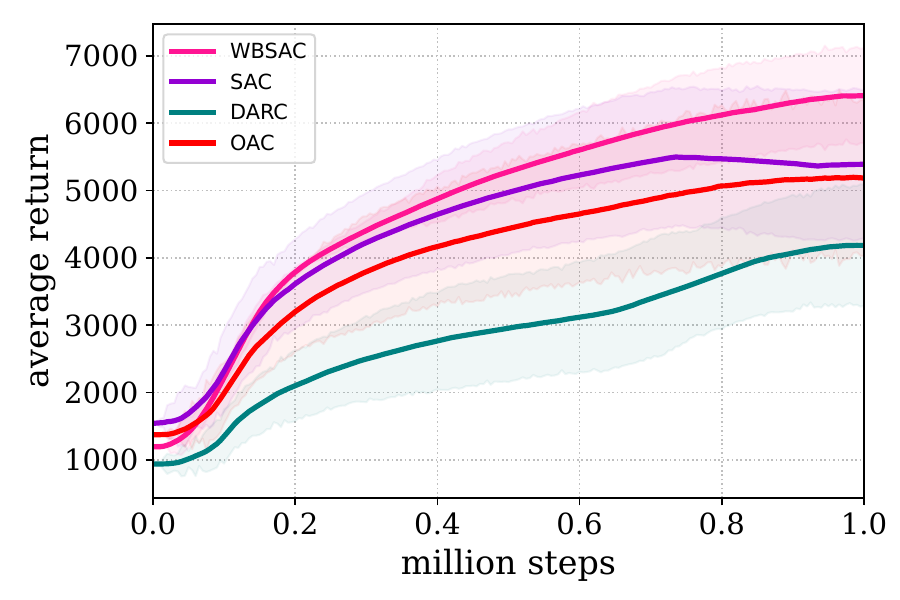}
    \caption{HalfCheetah}
    \label{fig:subfig-2}
  \end{subfigure}
  \hfill
  \begin{subfigure}{0.24\textwidth}
    \includegraphics[width=\linewidth]{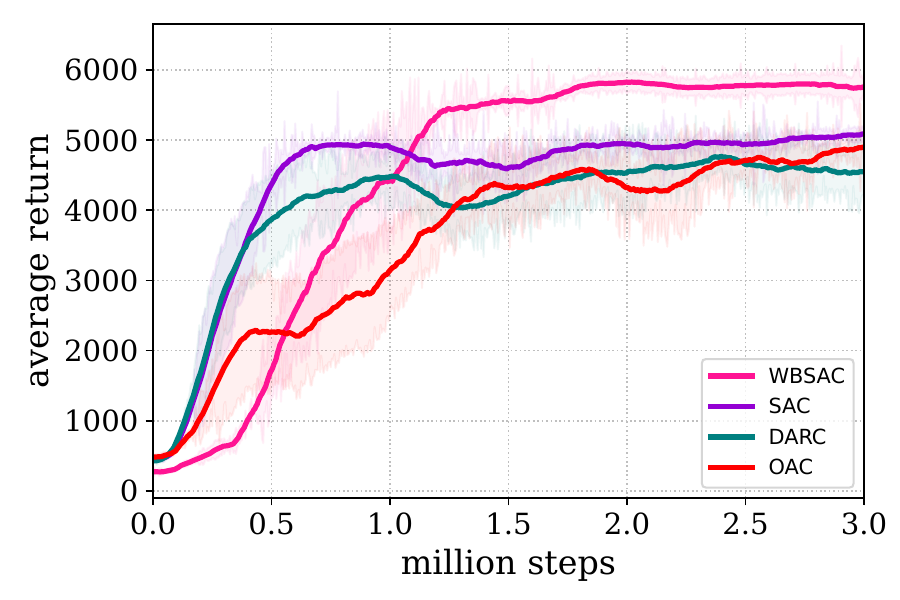}
    \caption{Humanoid}
    \label{fig:subfig-3}
  \end{subfigure}
  \hfill
  \begin{subfigure}{0.24\textwidth}
    \includegraphics[width=\linewidth]{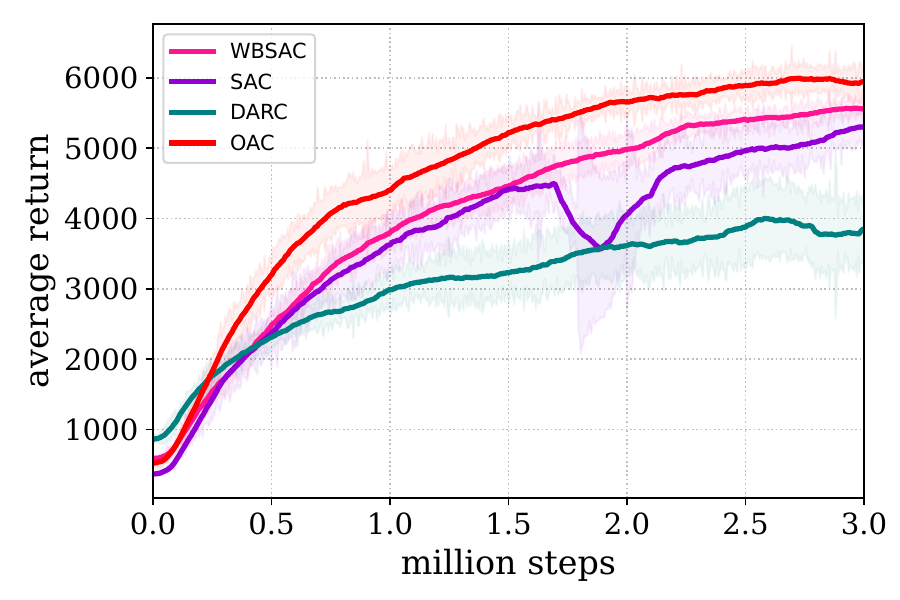}
    \caption{Ant}
    \label{fig:subfig-5}
  \end{subfigure}
  \caption{Performance comparison on MuJoCo environemnts. WBSAC outperforms SAC and DARC in all tasks and OAC in three.}
  \label{fig:return}
\end{figure*}

\begin{figure*}[!t]
  \centering
  % First row
  \begin{subfigure}{0.24\textwidth}
    \includegraphics[width=\linewidth]{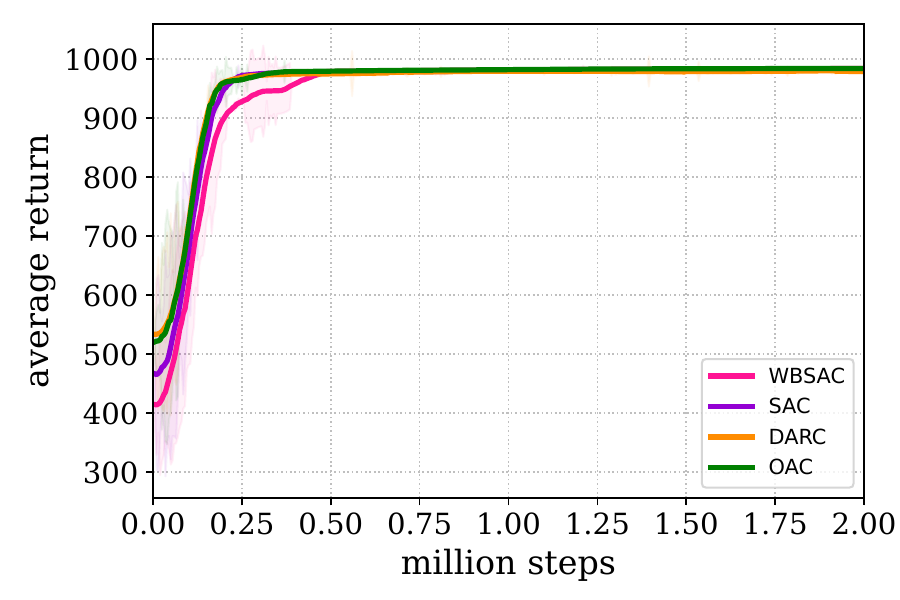}
    \caption{Ball in cup catch}
    \label{fig:subfig1}
  \end{subfigure}
  \hfill
  \begin{subfigure}{0.24\textwidth}
    \includegraphics[width=\linewidth]{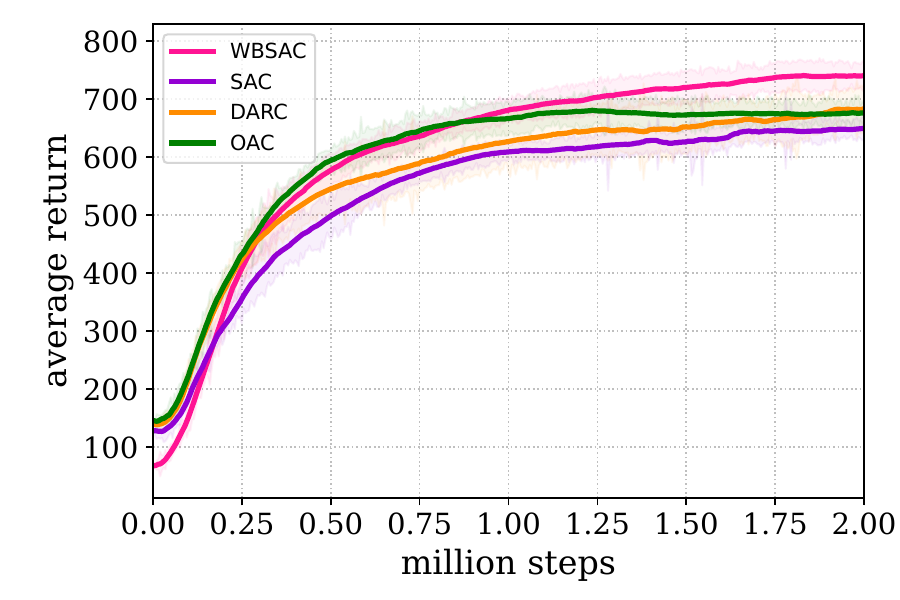}
    \caption{Cheetah-Run}
    \label{fig:subfig2}
  \end{subfigure}
  \hfill
  \begin{subfigure}{0.24\textwidth}
    \includegraphics[width=\linewidth]{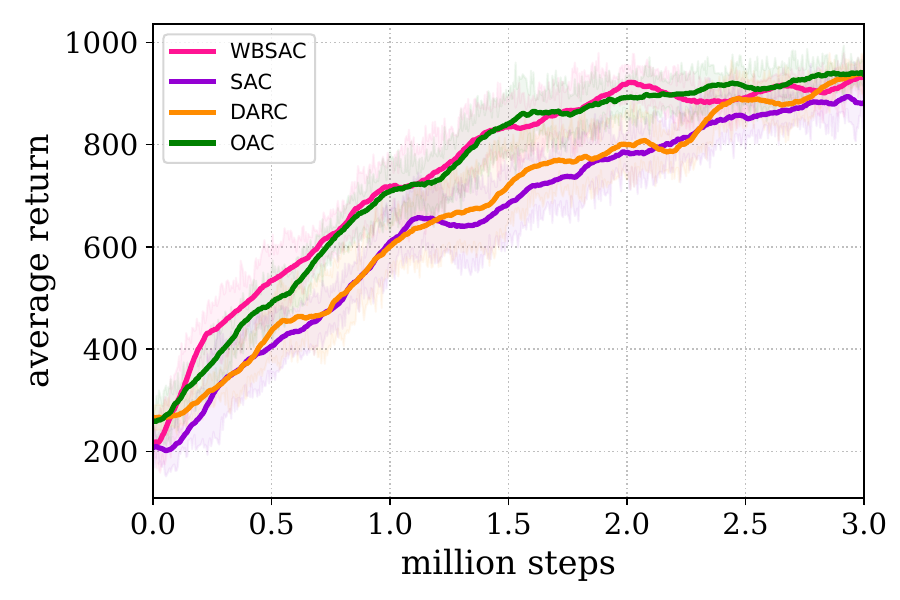}
    \caption{Finger turn-easy}
    \label{fig:subfig3}
  \end{subfigure}
  \hfill
  \begin{subfigure}{0.24\textwidth}
    \includegraphics[width=\linewidth]{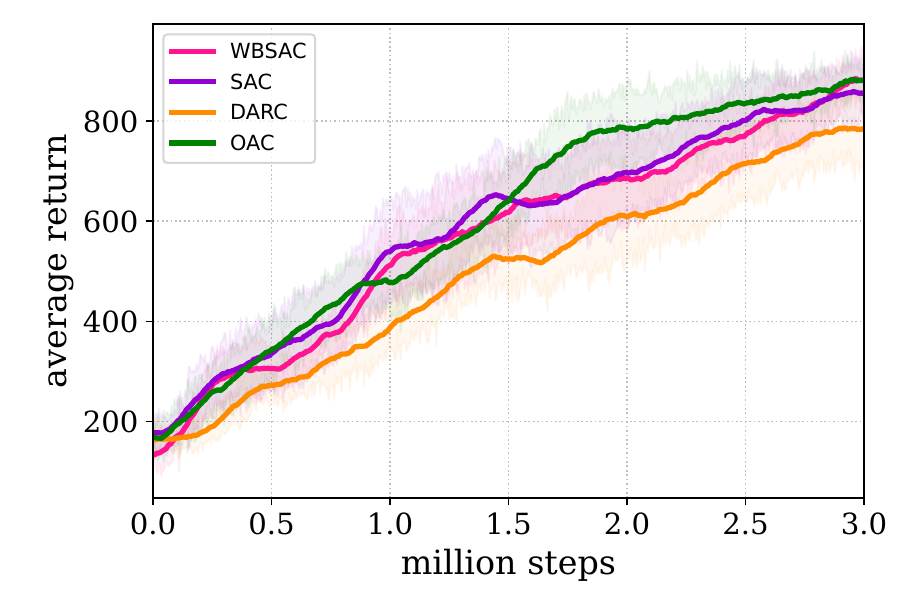}
    \caption{Finger turn-hard}
    \label{fig:subfig4}
  \end{subfigure}

  \vspace{0.4cm}

  % Second row
  \begin{subfigure}{0.24\textwidth}
    \includegraphics[width=\linewidth]{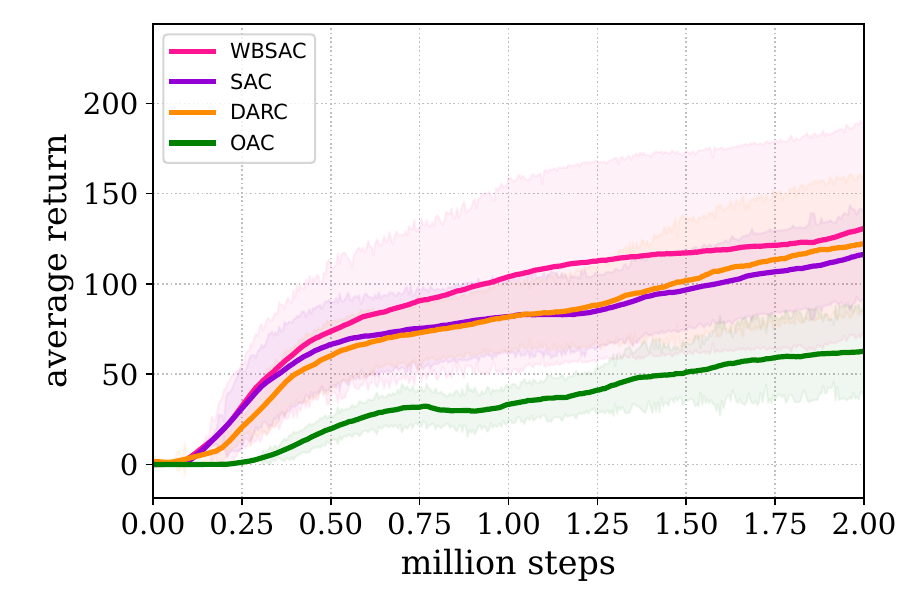}
    \caption{Hopper Hop}
    \label{fig:subfig5}
  \end{subfigure}
  \hfill
  \begin{subfigure}{0.24\textwidth}
    \includegraphics[width=\linewidth]{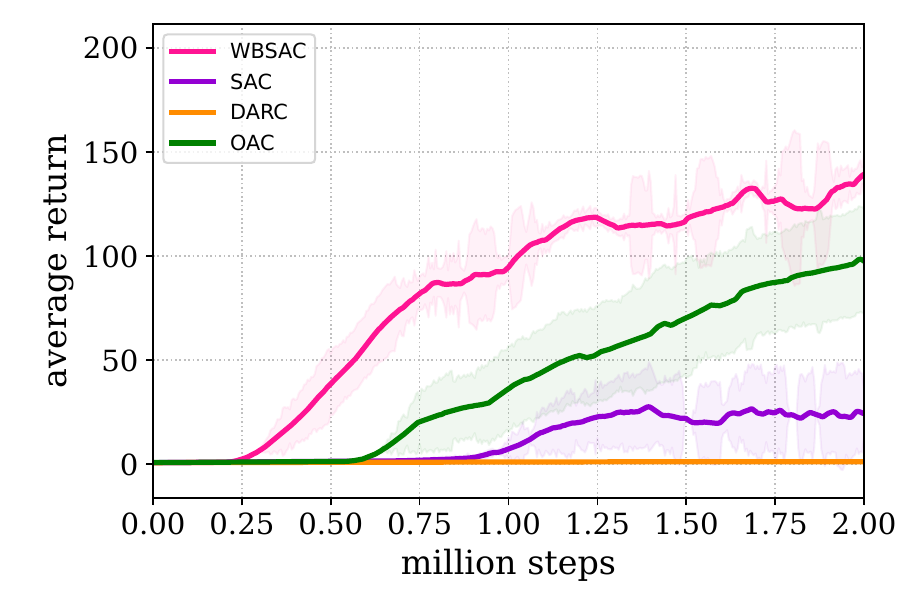}
    \caption{Humanoid-Run}
    \label{fig:subfig6}
  \end{subfigure}
  \hfill
  \begin{subfigure}{0.24\textwidth}
    \includegraphics[width=\linewidth]{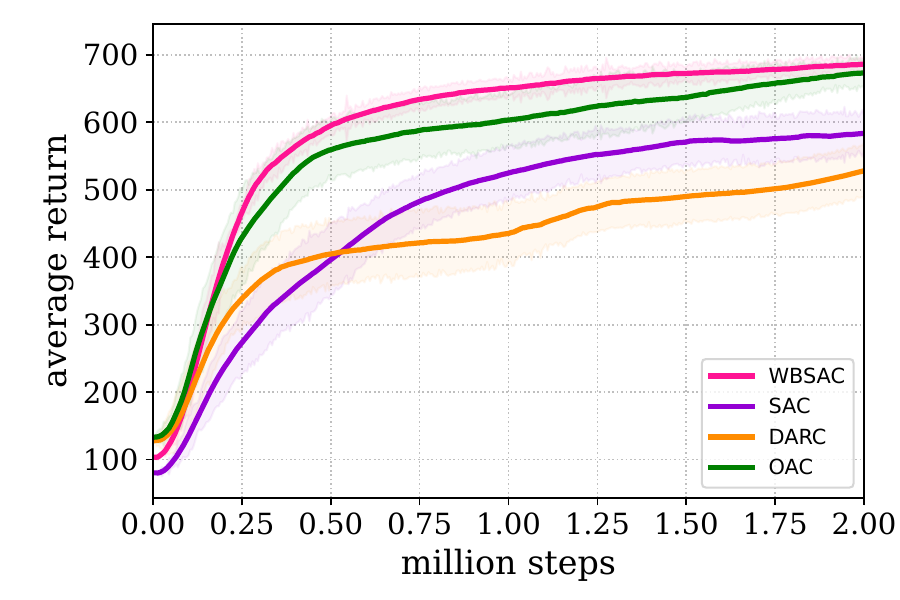}
    \caption{Walker-Run}
    \label{fig:subfig7}
  \end{subfigure}
  \hfill
  \begin{subfigure}{0.24\textwidth}
    \includegraphics[width=\linewidth]{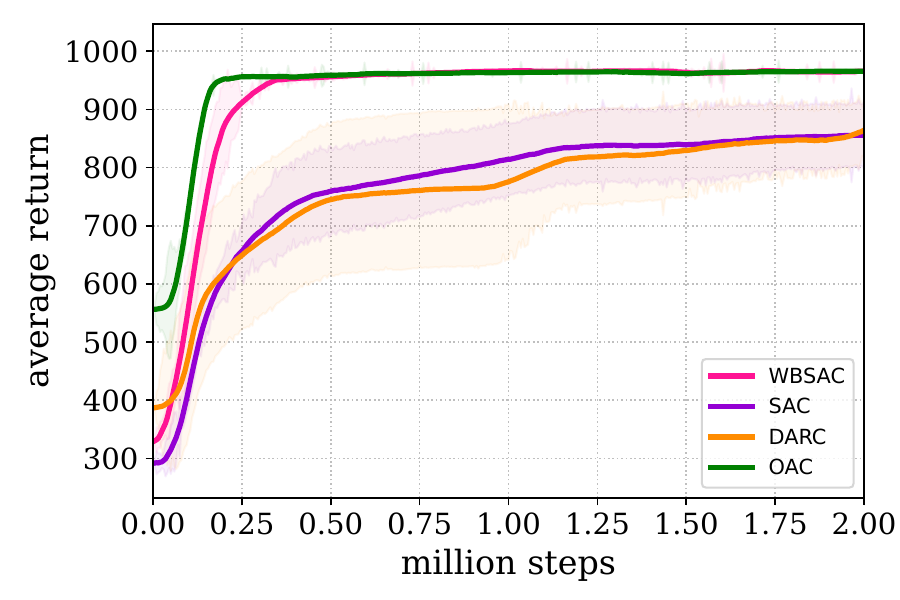}
    \caption{Walker-Walk}
    \label{fig:subfig8}
  \end{subfigure}

  \caption{Performance comparison on DeepMind control suite tasks. WBSAC consistently outperforms SAC and DARC, while it shows better or comparable performance with respect to OAC.}
  \label{fig:dmc}
\end{figure*}

\noindent\textbf{Performance on DeepMind control suite tasks. }We further assess WBSAC performance on seven challenging tasks from DeepMind control suite \cite{tassa2018deepmind} with both dense and sparse reward settings, which are known to be difficult for many off-policy, model-free reinforcement learning algorithms.

Figure~\ref{fig:dmc} depicts the average return and its standard deviation for five seeds. The results demonstrate that WBSAC outperforms the baselines and successfully solves the challenging humanoid-run task on which other methods fail. Table~\ref{tab:dmc_results} provides the average of last ten evaluation of the optimized policies.

\begin{figure}[!h]
    \centering
    \includegraphics[width=0.5\columnwidth]{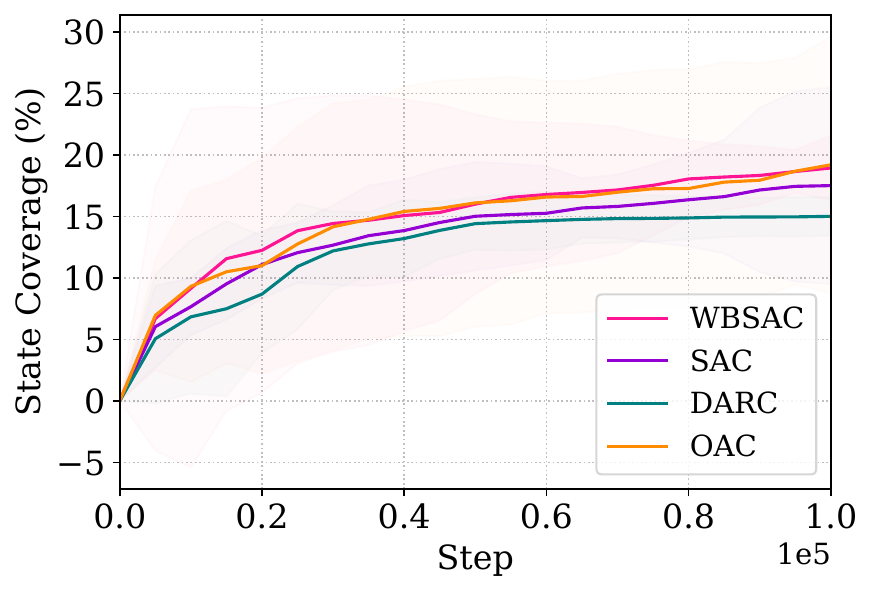}
    \caption{Average coverage over 3 seeds in the PointMaze Medium-v3 navigation task with sparse reward.}
    \label{fig:coverage}
\end{figure}

\noindent\textbf{Exploration capability. }To assess the exploration capabilities of WBSAC, we conducted tests in a navigation maze, which is a modified version of the PointMaze Medium-v3 environment \cite{de2023gymnasium}. In this PointMaze Medium-v3 setup, the agent's objective is to guide a ball to an unknown goal location within the maze. The ball starts in the maze's center, and we've designated two possible goal locations: the top-right and bottom-left corners. A sparse $0-1$ reward is applied upon reaching the goal. We train the agent for $100\text{k}$ steps, with its policy being evaluated every $5000$ steps. We present the average return and average coverage across $3$ different random seeds.

\begin{figure*}[!t]
  \centering
  % First row
  \begin{subfigure}{0.23\textwidth}
    \includegraphics[width=\linewidth]{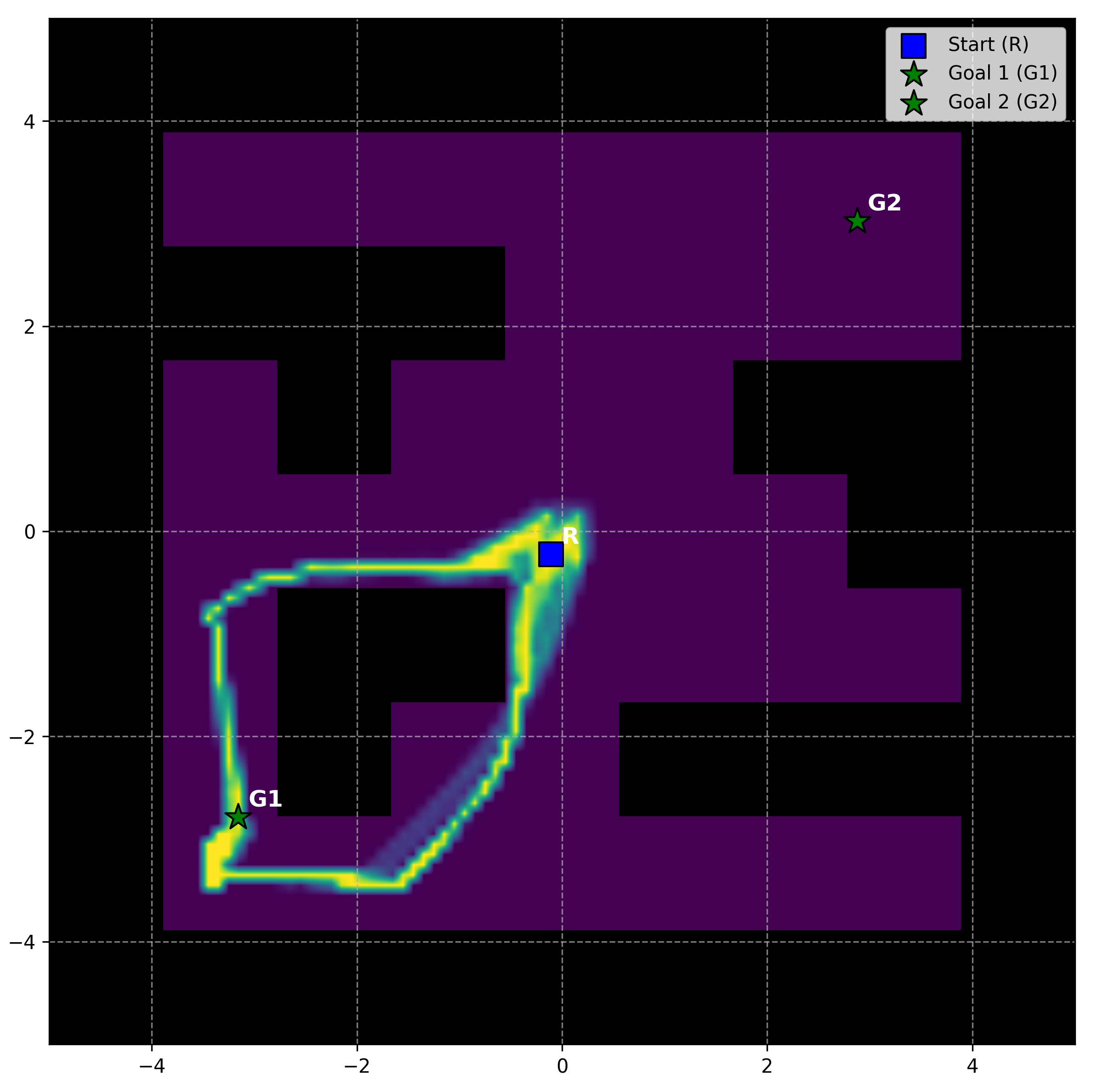}
    \caption{WBSAC}
    \label{fig:img1}
  \end{subfigure}\hfill
  \begin{subfigure}{0.23\textwidth}
    \includegraphics[width=\linewidth]{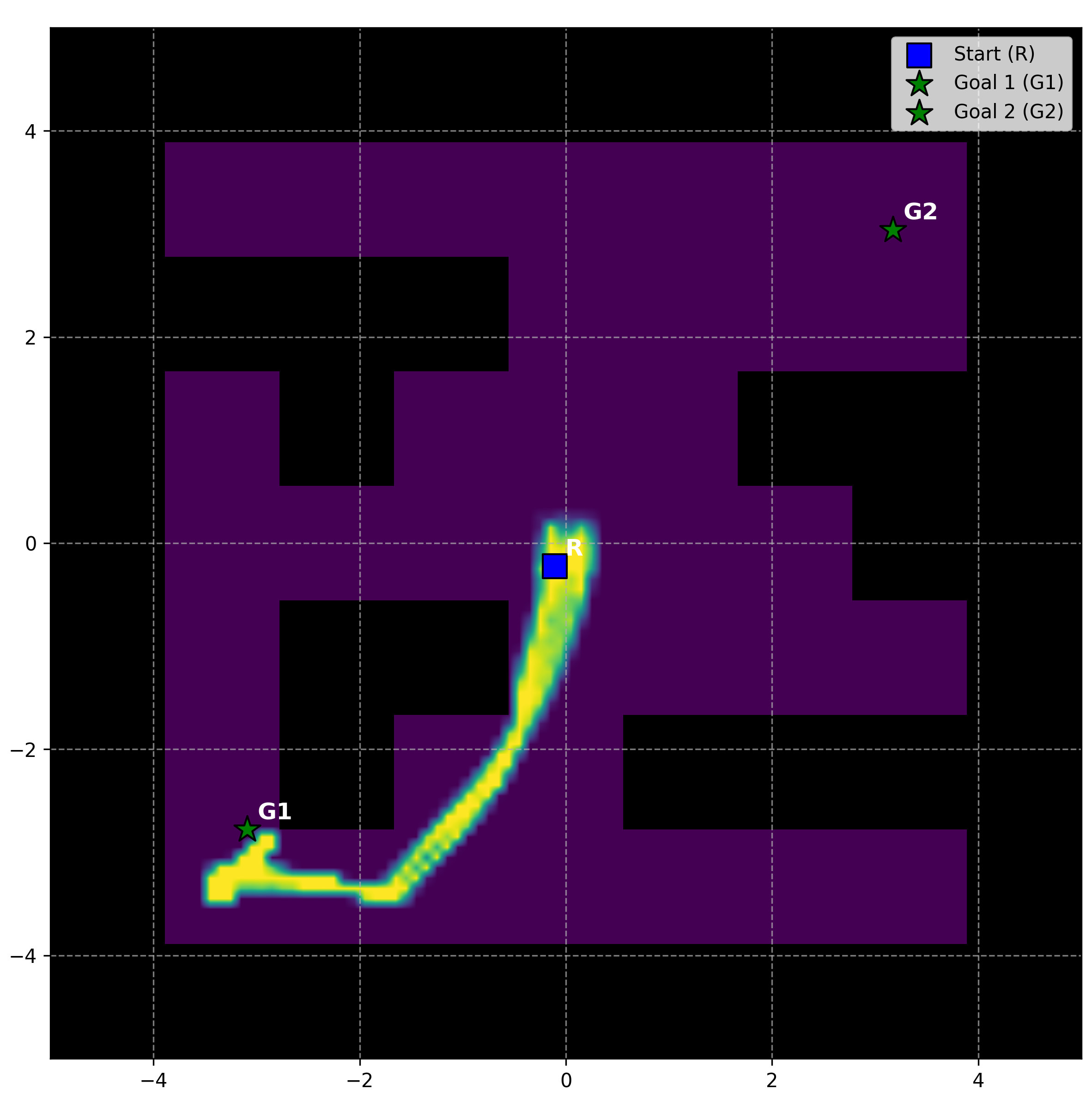} % WBSAC repeated
    \caption{SAC}
    \label{fig:img1b}
  \end{subfigure}\hfill
  \begin{subfigure}{0.23\textwidth}
    \includegraphics[width=\linewidth]{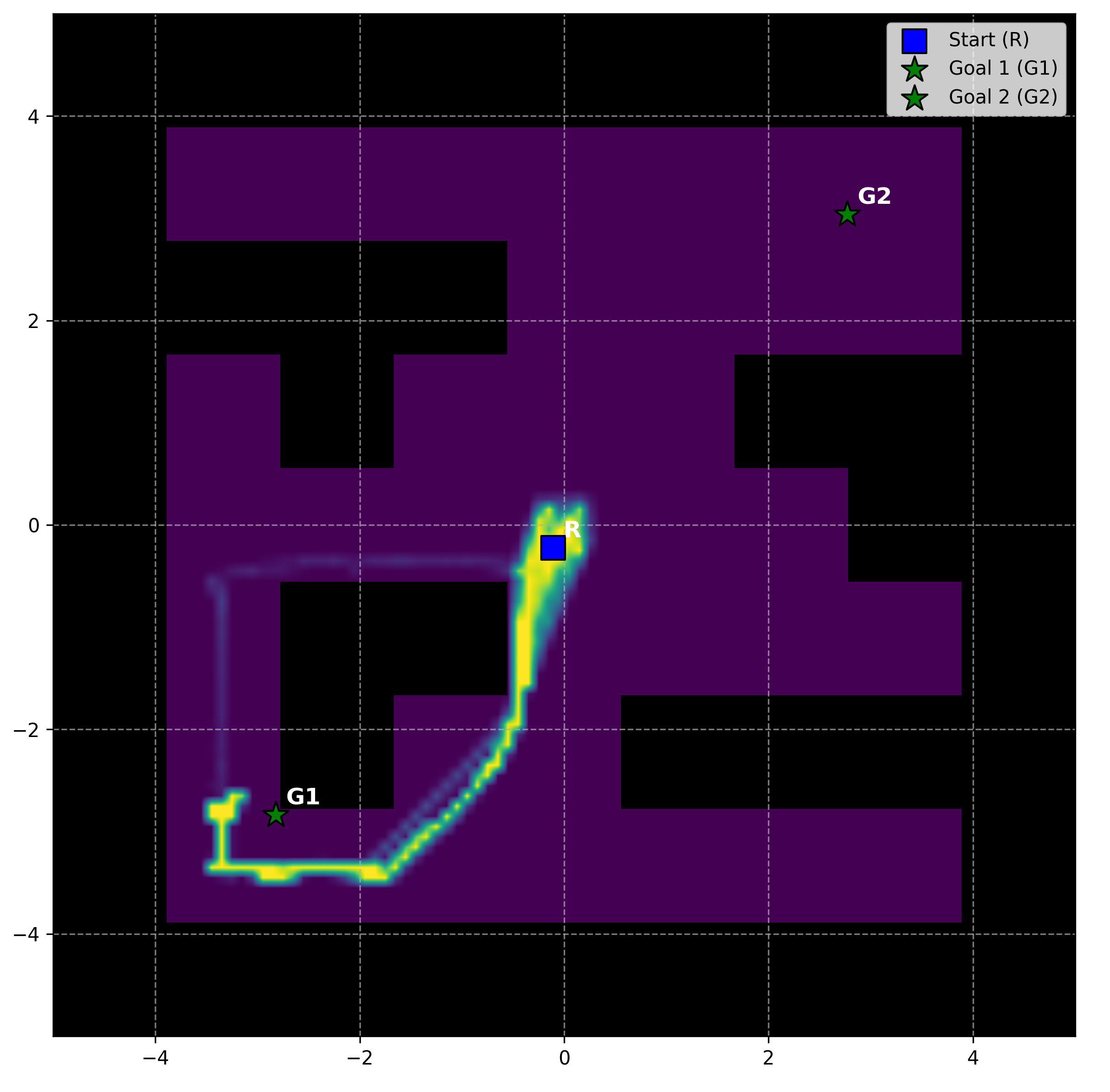}
    \caption{DARC}
    \label{fig:img2}
  \end{subfigure}\hfill
  \begin{subfigure}{0.23\textwidth}
    \includegraphics[width=\linewidth]{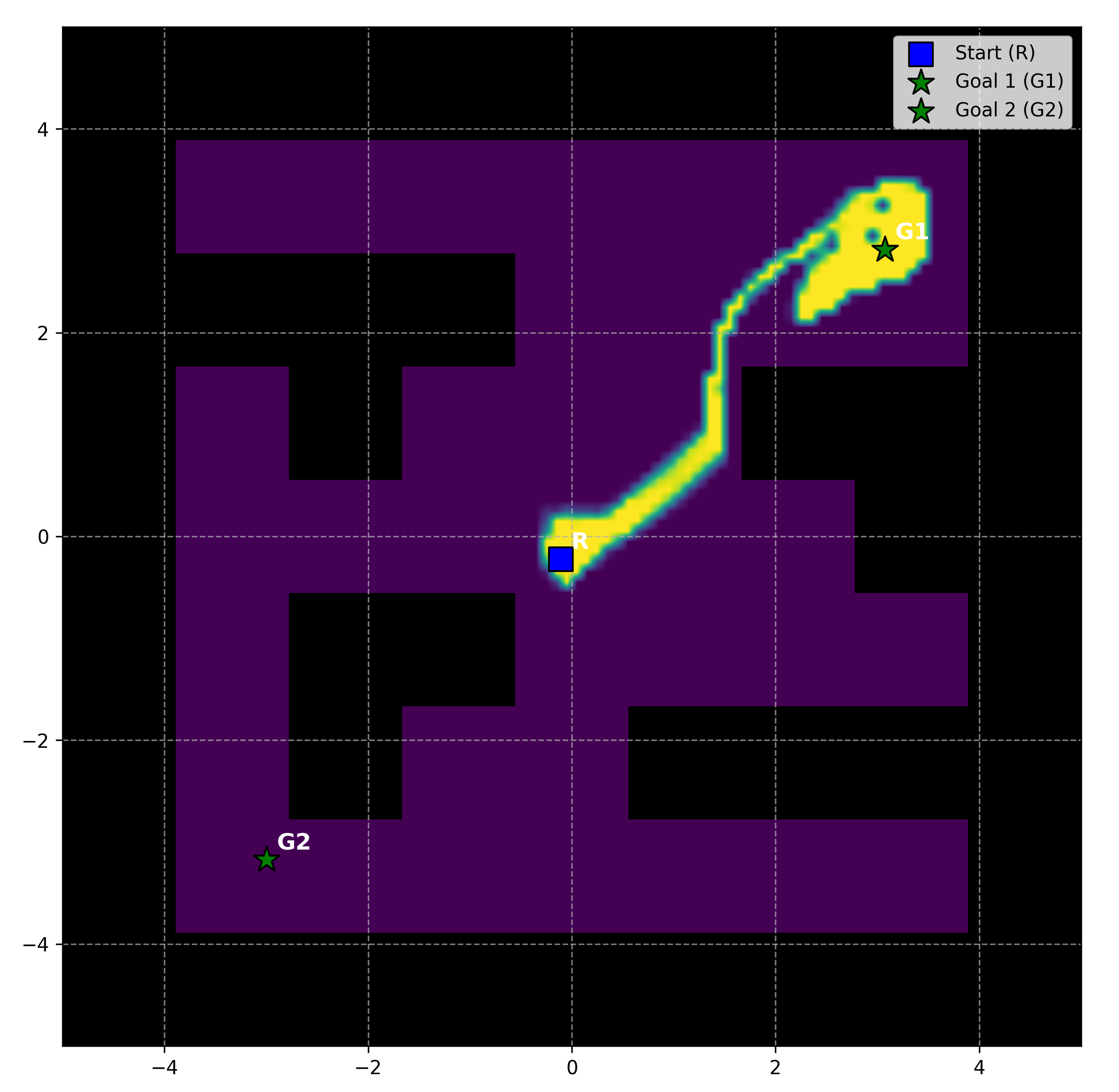}
    \caption{OAC}
    \label{fig:img3}
  \end{subfigure}

  \vspace{1em}

  % Second row
  \begin{subfigure}{0.23\textwidth}
    \includegraphics[width=\linewidth]{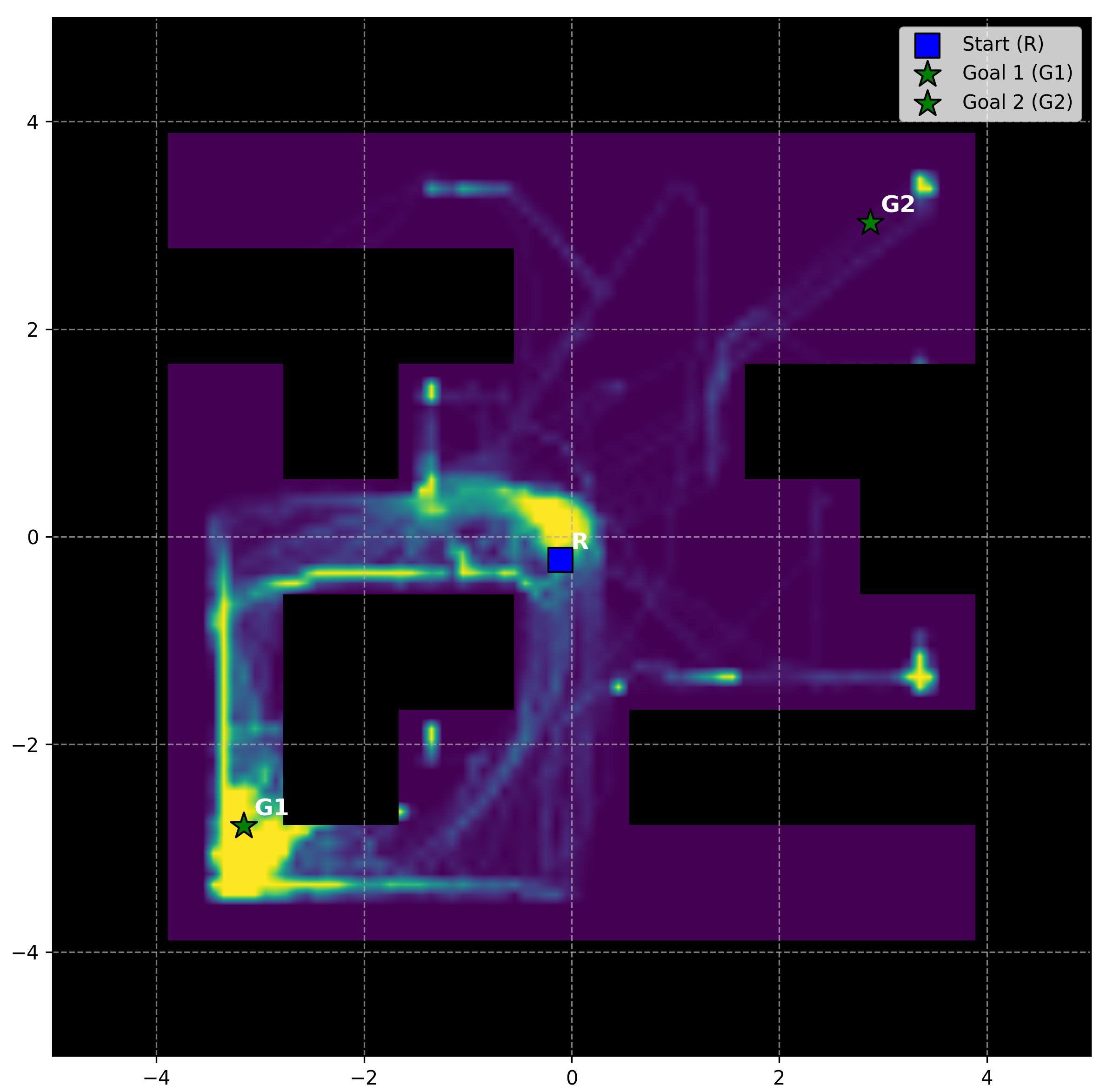}
    \caption{WBSAC}
    \label{fig:img4}
  \end{subfigure}\hfill
  \begin{subfigure}{0.23\textwidth}
    \includegraphics[width=\linewidth]{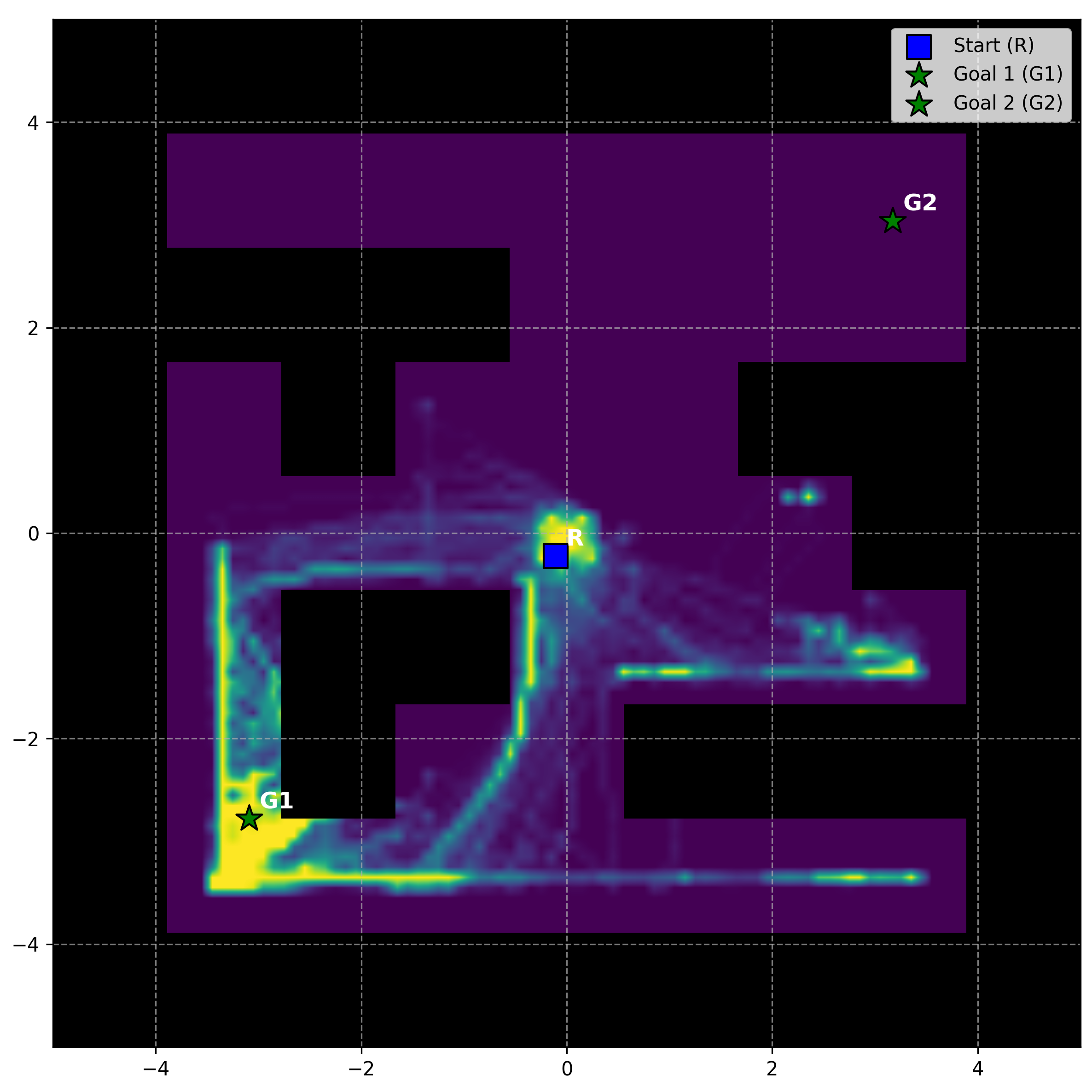} % WBSAC repeated
    \caption{SAC}
    \label{fig:img4b}
  \end{subfigure}\hfill
  \begin{subfigure}{0.23\textwidth}
    \includegraphics[width=\linewidth]{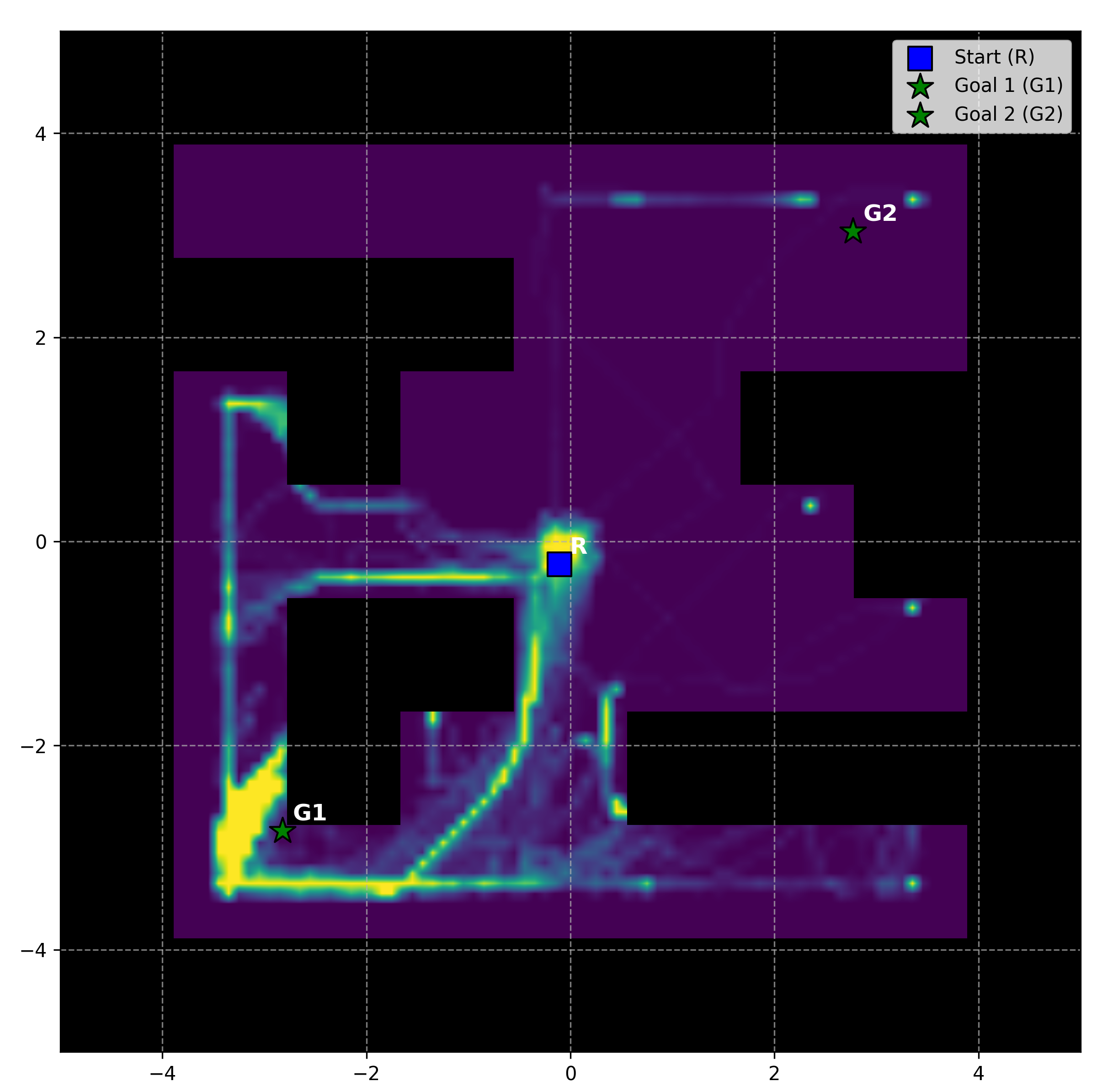}
    \caption{DARC}
    \label{fig:img5}
  \end{subfigure}\hfill
  \begin{subfigure}{0.23\textwidth}
    \includegraphics[width=\linewidth]{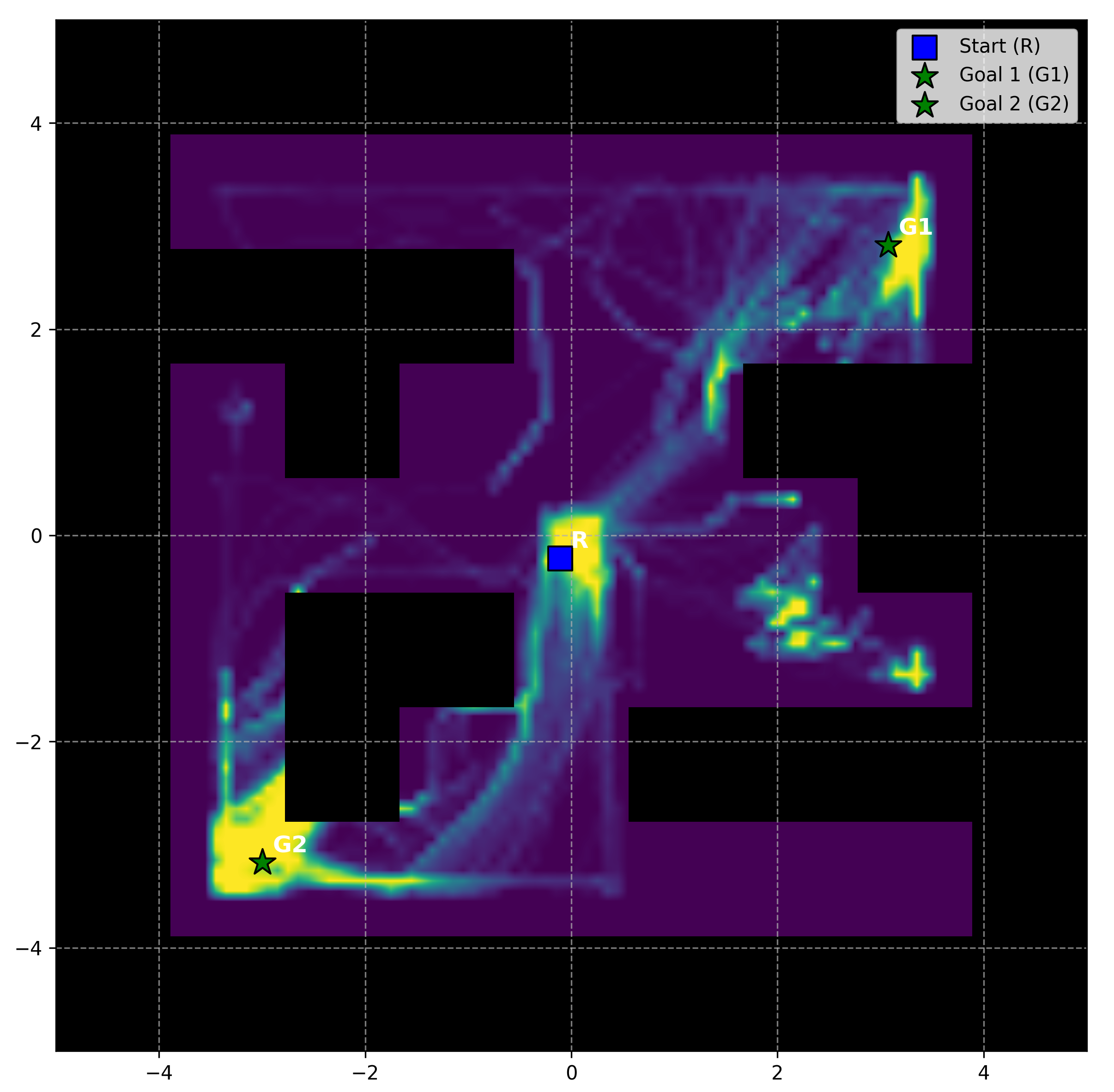}
    \caption{OAC}
    \label{fig:img6}
  \end{subfigure}

  \caption{State visitation heatmaps for pessimistic (first row) and exploration policies (second row).}
  \label{fig:eight-panel}
\end{figure*}

Figure~\ref{fig:coverage} indicates WBSAC achieves a higher coverage percentage than other baselines. Figure~\ref{fig:eight-panel} presents exploration heatmaps for the PointMaze task from one representative seed to visually evaluate exploration efficacy in this sparse reward setting. As shown in Figure~\ref{fig:eight-panel}, WBSAC exploration policy achieves more comprehensive state-space coverage compared to baselines. In comparison to baselines, WBSAC wastes less resources on visiting non-rewarding regions of the maze, and focus its exploration on promising areas, which is because of directed exploration. The heatmap of our pessimistic policy shows that WBSAC successfully identifies multiple distinct paths to the goal, which indicates a more robust understanding of the task.

\begin{figure}[!h]
    \centering
    \begin{subfigure}{0.49\columnwidth}
        \centering
        \includegraphics[width=\linewidth]{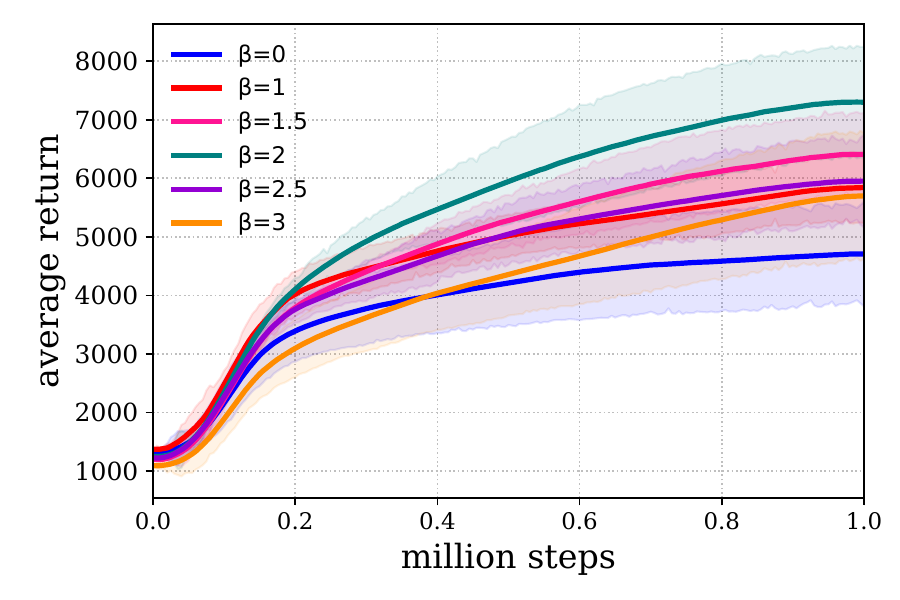}
        \caption{Sensitivity to $\beta_o$.}
        \label{fig:cheetahLambda_app}
    \end{subfigure}\hfill
    \begin{subfigure}{0.49\columnwidth}
        \centering
        \includegraphics[width=\linewidth]{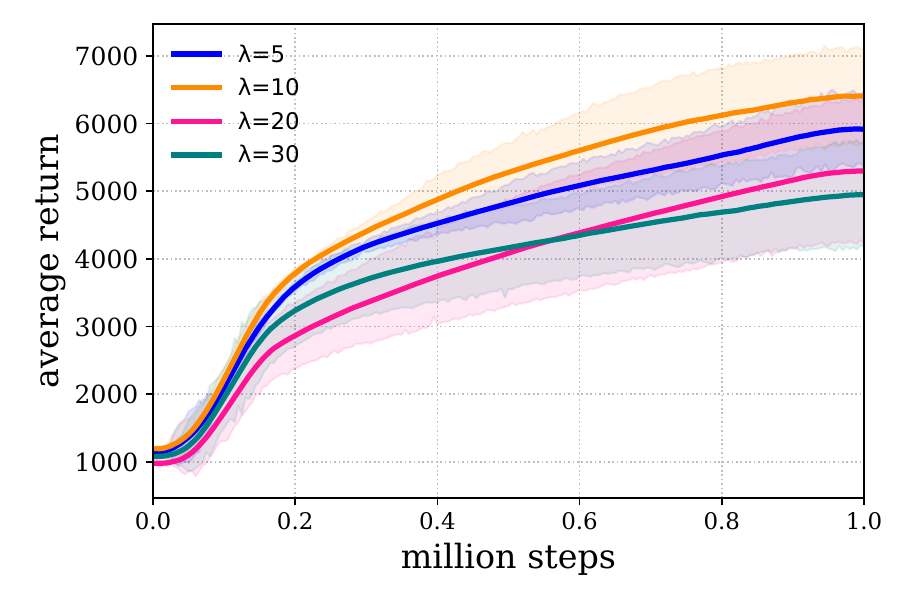}
        \caption{Sensitivity to $\lambda$.}
        \label{fig:betaCheetah_app}
    \end{subfigure}
    \caption{Sensitivity analysis of WBSAC to hyperparameters in HalfCheetah-v5 domain.}
    \label{fig:sensitivity2_app}
\end{figure}

\noindent\textbf{Hyperparameter sensitivity.} WBSAC relies on two hyperparameters, $\lambda$, which determines the transition rate from pessimistic strategy to optimistic behavior, and $\beta_o$, which controls the amount
of uncertainty used to compute the optimistic policy. Figure~\ref{fig:cheetahLambda_app} and Figure~\ref{fig:betaCheetah_app} demonstrate that the degree of optimism and the exploration schedule parameter affect the agent's performance. For very low values of $\lambda$ the policy remains pessimistic for longer, which slows down exploration and convergence to optimal behavior. Conversely, very high values of $\lambda$ lead to a rapid shift towards the optimistic policy and can cause unstable learning or convergence to poor local optima. Finally, removing the influence of critic disagreement on the optimistic actor's objective by setting the optimistic degree $\beta_o=0$ leads to a noticeable degradation in WBSAC's performance.

\section{Conclusion, limitations and future work}
\label{sec: Conclusion}

We present the WBSAC framework, a model-free deep reinforcement algorithm that uses dual actors. We use the Wasserstein barycenter of two pessimistic and optimistic policies to make a balance between exploration and exploitation. By adjusting the hyperparameters, WBSAC starts with a pessimistic strategy and gradually shifts to an optimistic behavior. This enhances sample efficiency and facilitates a smooth exploration process. This is validated through several experiments. 

Despite its promising performance, the WBSAC framework presents a few limitations. Using a dual actor network setup leads to increased memory and computational cost compared to the state-of-the-art off-policy actor-critic algorithms. Furthermore, for policies represented by Gaussian distributions, the barycenter can often be derived in a straightforward, closed-form analytical solution. However, for more complex, non-Gaussian policy representations, the computation of the Wasserstein barycenter typically lacks a closed form and requires iterative numerical approaches. A future direction for this work is to compensate for this cost by focusing on the effective sampling experience replay buffer data, potentially with replay ratio scaling.

% \section*{References}

\renewcommand*{\bibfont}{\footnotesize}
\bibliographystyle{plainnat}
\bibliography{main}

%%%%%%%%%%%%%%%%%%%%%%%%%%%%%%%%%%%%%%%%%%%%%%%%%%%%%%%%%%%%

\newpage
\appendix

\begin{center}
\textbf{Wasserstein Barycenter Soft Actor-Critic}
\end{center}

\section{Further discussion and broader impacts}

The WBSAC exploration policy entropy lower bound ensures that the transition from pessimistic to optimistic behavior maintains a minimum level of action diversity. This enhanced state-action coverage has practical implications beyond immediate performance gains. For instance, an agent trained on a wider variety of experiences is less likely to overfit to a narrow subset of the environment. This enables better generalization to new tasks and enhances adaptability in unfamiliar environments, which are essential for effective real-world deployment. WBSAC's mechanism of beginning with a pessimistic policy and gradually incorporating optimism could also offer a framework for safer exploration. The initial conservative behavior might help avoid catastrophic failures during early learning phases, especially in safety-critical systems.

\section{Proof of proposition 1}

\begin{proposition}For factorized Gaussian pessimistic and optimistic policies, the exploration policy $\pi_e$ (derived from (\ref{eq:barycenter_mean_wbsac}) and (\ref{eq:barycenter_cov_wbsac}) has its differential entropy, $H(\pi_e(s))$, lower-bounded for any state $s \in \mathcal{S}$ as:
\begin{equation}
\label{eq:entropy_inequality}
H\!\bigl(\pi_e(s)\bigr)\;\;\ge\;\;
\xi_p\,H\!\bigl(\pi_p(s)\bigr)+\xi_o\,H\!\bigl(\pi_o(s)\bigr).
\end{equation}
\end{proposition}
\begin{proof}
Consider $d$-dimensional Gaussian pessimistic policy \(\pi_p(\cdot|s) = \mathcal{N}(\mu_p(s), \Sigma_p(s))\), and optimistic policy \(\pi_o(\cdot|s) = \mathcal{N}(\mu_o(s), \Sigma_o(s))\), with $\Sigma_p(s) = \operatorname{diag}(\sigma_{p,1}^2(s), \ldots, \sigma_{p,d}^2(s)$, and $\Sigma_o(s) = \operatorname{diag}(\sigma_{o,1}^2(s), \ldots, \sigma_{o,d}^2(s))$, where \(\sigma_{p,i}^2(s)\) and \(\sigma_{o,i}^2(s)\) are the variances along dimension \(i\). Considering the differential entropy definition as $H(\pi(s)) = \frac{1}{2} \ln \det (2 \pi e \Sigma(s))$ for $\pi(s)$, we start by pessimistic policy. Since \(\Sigma_p(s)\) is diagonal, its square root is \(\Sigma_p(s)^{1/2} = \operatorname{diag}(\sigma_{p,1}(s), \ldots, \sigma_{p,d}(s))\). Hence, 

\begin{equation}
\det \Sigma_p(s) = \prod_{i=1}^d (\sigma_{p,i}(s))^2
\end{equation}
and, 
\begin{equation}
\ln \det \Sigma_p(s) = 2 \sum_{i=1}^d \ln ( \sigma_{p,i}(s))
\end{equation}
Therefore, the entropy becomes:
\begin{equation}
H(\pi_p(s)) = \frac{1}{2} \ln \det (2 \pi e \Sigma_p(s)) = \frac{1}{2} \ln (2 \pi e)^d + \sum_{i=1}^d \ln \sigma_{p,i}(s)
\end{equation}

Similarly, $H(\pi_o(s))  = \frac{1}{2} \ln (2 \pi e)^d + \sum_{i=1}^d \ln \sigma_{o,i}(s)$. For the exploration policy \(\pi_e(\cdot|s) = \mathcal{N}(\mu_e(s), \Sigma_e(s))\), with covariance $\Sigma_e(s)$ derived from (\ref{eq:barycenter_cov_wbsac}), we have:
\begin{equation}
    \xi_p \Sigma_p(s)^{1/2} + \xi_o \Sigma_o(s)^{1/2} = \operatorname{diag}(\xi_p \sigma_{p,1}(s) + \xi_o \sigma_{o,1}(s), \ldots, \xi_p \sigma_{p,d}(s) + \xi_o \sigma_{o,d}(s))
\end{equation}

Therefore, $\Sigma_e(s) = \operatorname{diag}\left( (\xi_p \sigma_{p,1}(s) + \xi_o \sigma_{o,1}(s))^2, \ldots, (\xi_p \sigma_{p,d}(s) + \xi_o \sigma_{o,d}(s))^2 \right)$ and the entropy $H(\pi_e(s))$ becomes:

\begin{equation}
H(\pi_e(s)) = \frac{1}{2} \ln (2 \pi e)^d + \sum_{i=1}^d \ln (\xi_p \sigma_{p,i}(s) + \xi_o \sigma_{o,i}(s))
\end{equation}

To proof inequality \ref{eq:entropy_inequality}, consider:

\begin{equation}
\xi_p \sum_{i=1}^d \ln \sigma_{p,i}(s) + \xi_o \sum_{i=1}^d \ln \sigma_{o,i}(s)=\sum_{i=1}^d \left( \xi_p \ln \sigma_{p,i}(s) + \xi_o \ln \sigma_{o,i}(s) \right)
\end{equation}

Since the \(\ln(x)\) is concave for \(x > 0\), Jensen's inequality, for non-negative $\xi_p$, $\xi_o$ such that \(\xi_p + \xi_o = 1\), gives: 

\begin{equation}
\sum_{i=1}^d \left( \xi_p \ln \sigma_{p,i}(s) + \xi_o \ln \sigma_{o,i}(s) \right) \leq \sum_{i=1}^d \ln (\xi_p \sigma_{p,i}(s) + \xi_o \sigma_{o,i}(s))
\label{inequality}
\end{equation}

Because \(\xi_p + \xi_o = 1\), the constant term can be decomposed as:

\begin{equation}
\frac{1}{2} \ln (2 \pi e)^d = \xi_p \cdot \frac{1}{2} \ln (2 \pi e)^d + \xi_o \cdot \frac{1}{2} \ln (2 \pi e)^d
\label{constant_term}
\end{equation}

Adding (\ref{constant_term}) to each side of (\ref{inequality}), we obtain:

\begin{equation} \label{eq:entropy_inequality_fully_annotated} 
\begin{split}
& \xi_p \textcolor{black}{\underbrace{\left( \frac{1}{2} \ln\left((2\pi e)^d\right) + \sum_{i=1}^d \ln \sigma_{p,i}(s) \right)}_{\textcolor{blue}{H(\pi_p(s))}}}
+ \xi_o \textcolor{black}{\underbrace{\left( \frac{1}{2} \ln\left((2\pi e)^d\right) + \sum_{i=1}^d \ln \sigma_{o,i}(s) \right)}_{\textcolor{blue}{H(\pi_o(s))}}} \\ 
& \le \textcolor{black}{\underbrace{\frac{1}{2} \ln\left((2\pi e)^d\right) + \sum_{i=1}^d \ln\left(\xi_p \sigma_{p,i}(s) + \xi_o \sigma_{o,i}(s)\right)}_{\textcolor{blue}{H(\pi_e)}}}
\end{split}
\end{equation}

% \begin{equation}
% \begin{split}
% \xi_p \left( \frac{1}{2}(2 \pi e)^d + \sum_{i=1}^d \ln \sigma_{p,i}(s) \right) + \xi_o \left( \frac{1}{2} \ln (2 \pi e)^d + \sum_{i=1}^d \ln \sigma_{o,i}(s) \right) \\ \leq \frac{1}{2} \ln (2 \pi e)^d + \sum_{i=1}^d \ln (\xi_p \sigma_{p,i}(s) + \xi_o \sigma_{o,i}(s))
% \end{split}
% \end{equation}

Hence, 
\begin{equation}
\xi_p H(\pi_p(s)) + \xi_o H(\pi_o(s)) \leq  H(\pi_e(s))
\end{equation}

\end{proof}

\section{Experimental details}

For a fair comparison with baselines, we set the initial seeds of the computational packages and fix the seeds of the environments to ensure reproducibility of results given initial seed values. Additionally, instead of randomly sampling training and evaluation seeds, we force the seeds to come from disjoint sets \cite{ciosek2019better}. For WBSAC, all networks share the same architecture. We use two hidden layers of $256$ units, and the activation function is ReLU. All networks are trained with Adam optimizer \cite{kingma2014adam} with a learning rate of 3e-4. The hyperparameters of WBSAC are tuned specifically for the Ant-v5 case study using grid search over $\beta_o \in \{1,1.5,2,2.5,3,4\}$, $\lambda \in \{2, 6,10, 12, 15,20\}$. Hyperparameters were kept constant across tasks. All the hyperparameters of SAC, DARC, and OAC baselines are reported in Table. ~\ref{tab:hyperparams}. The results for our SAC baseline were generated using this publicly available PyTorch implementation from \url{https://github.com/denisyarats/pytorch_sac}. To ensure fair comparison, the hyperparameters for SAC baseline were configured as specified in the original SAC paper \cite{haarnoja2018soft}. For the DARC baseline, we use the source code available at \url{https://github.com/dmksjfl/SMR/blob/master/DARC.py}, and for reproducing
the OAC results we used \url{https://github.com/microsoft/oac-explore}. All experiments were conducted on NVIDIA A100-PCIE GPU with 40GB of RAM.

\begin{table}[!h]
\centering
\caption{Hyperparameters used across MuJoCo and DeepMind control suite Tasks.}
\label{tab:hyperparams}
\resizebox{0.7\columnwidth}{!}{%
\begin{tabular}{lcccc}
\toprule
\textbf{Hyperparameter} & \textbf{WBSAC} & \textbf{SAC}  & \textbf{DARC} & \textbf{OAC} \\
\midrule
Number of hidden layers            & 2               & 2       & 2       & 2    \\
Number of hidden nodes             & 256             & 256     & 256     & 256  \\
Activation                         & ReLU            & ReLU    & ReLU    & ReLU \\
Batch size                         & 256             & 256     & 256     & 256  \\
Replay buffer size                 & $10^6$          & $10^6$  & $10^6$  & $10^6$ \\
Discount factor                    & 0.99            & 0.99    & 0.99    & 0.99 \\
Target smoothing                   & 0.005           & 0.005   & 0.005   & 0.005\\
Optimizer                          & Adam            & Adam    & Adam    & Adam \\
Actor learning rate                & $3\times10^{-4}$ & $3\times10^{-4}$ & $3\times10^{-4}$ & $3\times10^{-4}$ \\
Critic learning rate               & $3\times10^{-4}$ & $3\times10^{-4}$ & $3\times10^{-4}$ & $3\times10^{-4}$ \\
Entropy coefficient                & $0.2$           & $0.2$   & N/A     & $0.2$  \\
Maximum log std                    & $2$             & $2$     & N/A     & $2$  \\
Minimum log std                    & $-20 $          & $-20 $  & N/A     & $-20$  \\
Regularization parameter           & N/A             & N/A     & $0.005$ & N/A\\
Noise clip                         & N/A             & N/A     & $0.5$   & N/A \\
Exploration noise                  & N/A             & N/A     & $\mathcal{N}(0, 0.1)$ & N/A  \\
Upper bound uncertainty coefficient      & $1.5$           & N/A     & N/A     & 4.36  \\
Exploration schedule $\lambda$     & $10$            & N/A     & N/A     & N/A  \\
Shift multiplier $\sqrt{2\delta}$ & N/A             & N/A     & N/A     & $3.69$  \\
\bottomrule
\end{tabular}}
\end{table}

\section{Ablation study}
\noindent\textbf{Dynamic and fixed $\xi_o$.} To evaluate the impact of our dynamic exploration mechanism, we provide a comparison between WBSAC using its variable weighting scheme for the Wasserstein barycenter and configurations that employ fixed weights. As can be seen clearly in Figure~\ref{fig:variable_weights_app}, our adaptive approach for varying the weights of the pessimistic and optimistic policies achieves better performance in comparison to configurations that bias the exploration policy to the pessimistic policy via a high $\xi_p$.

\begin{figure}[!h]
    \centering
    \includegraphics[width=0.5\columnwidth]{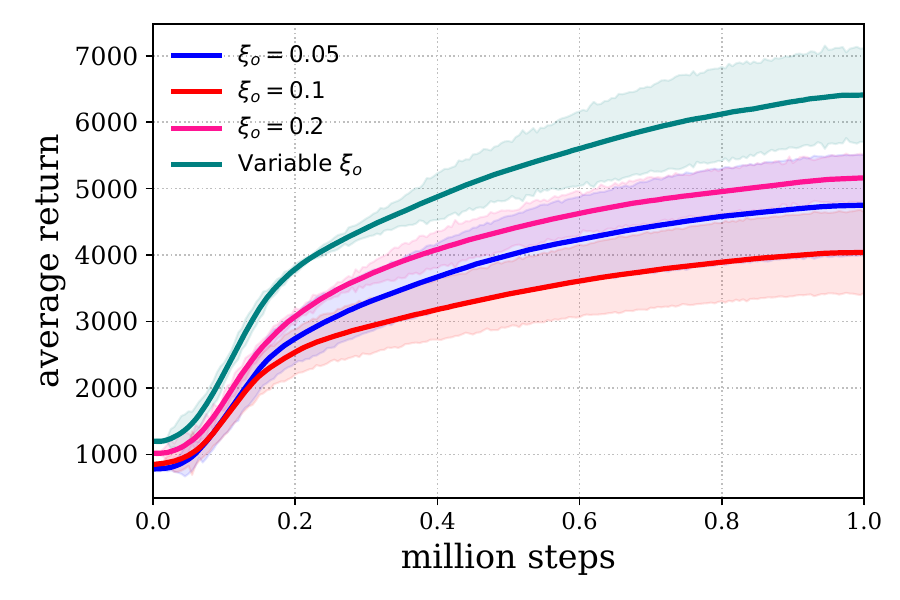}
    \caption{Ablation study on WBSAC performance under fixed and adaptive $\xi_o$.}
    \label{fig:variable_weights_app}
\end{figure}

\begin{figure}[!h]
  \centering
  \includegraphics[width=0.5\columnwidth]{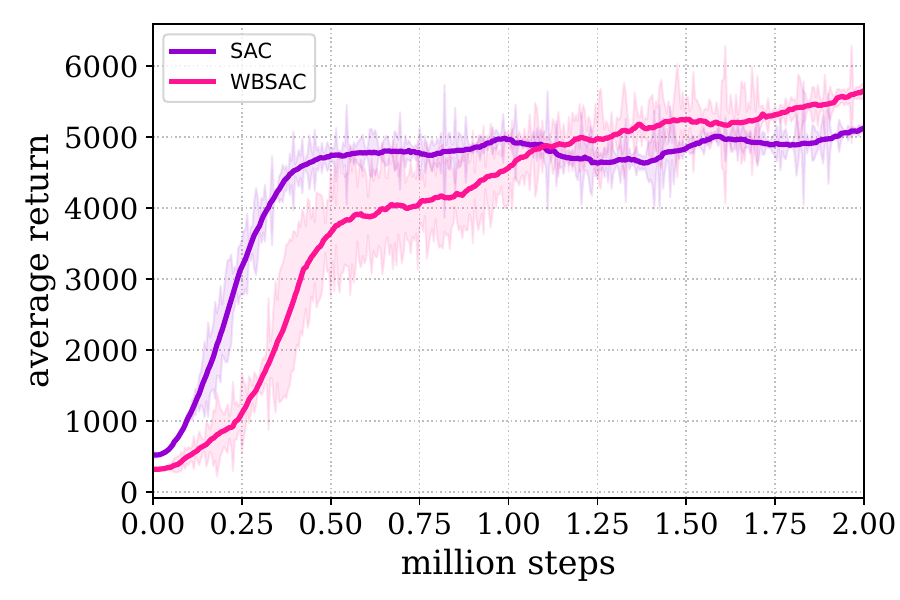}
  \caption{Performance comparison of WBSAC and SAC with $4$ gradient steps per environment interaction on the Humanoid task.}
  \label{fig:LUTD_app}
\end{figure}

\noindent\textbf{Performance with more training steps. }We further investigate the impact of increased update-to-data ratios (gradient steps per environment step). Figure~\ref{fig:LUTD_app} compares the performance of WBSAC and SAC with $4$ gradient steps per environment step, averaged over three seeds on the Humanoid task. WBSAC ultimately achieves a higher final average return than SAC under these experimental conditions.

\end{document}